\documentclass[11pt]{article}

 \title{On implicit regularization: Morse functions and applications to matrix factorization}
\author{Mohamed Ali Belabbas}
\usepackage{amssymb}
\usepackage{latexsym}
\usepackage{amsmath}
\usepackage{amsthm}
\usepackage{amscd}
\usepackage{mathtools}
\mathtoolsset{showonlyrefs}
\usepackage{lmodern}
\usepackage{tikz}

\usepackage{cite} 
\usepackage{setspace}
\usepackage[latin1]{inputenc}
\usepackage{amsmath}
\usepackage{amsthm}
\usepackage{color}
\usepackage{amsfonts}
\usepackage{amssymb}
\usepackage{graphicx}
\usepackage{epsfig}
\usepackage{url}
\usepackage{tikz,pgfplots}
\usepackage{tikz-cd} 
\usepackage{bbm} 
\usepackage{microtype}

\usetikzlibrary{external}
\newtheorem{theorem}{Theorem}

\newtheorem{proposition}[theorem]{Proposition}
\newtheorem{corollary}[theorem]{Corollary}
\newtheorem{lemma}[theorem]{Lemma}

\newtheorem{remark}{Remark}
\newtheorem{exampleA}{Example}

\def\qed{ \rule{.08in}{.08in}}

\newcommand{\diag}{\operatorname{diag}}

\newcommand{\R}{\mathbb{R}}

\newcommand{\0}{\mathbf{0}}

\newcommand{\tr}{\operatorname{tr}}
\newcommand{\Span}{\operatorname{span}}

\newcommand{\grad}{\operatorname{grad}}
\newcommand{\crit}{\operatorname{Crit}}

\newcommand{\oR}{{ \mathbf R}}

\newcommand{\oT}{\mathrm{T}}

\newcommand{\cD}{\mathcal{D}}

\newcommand{\feas}{\operatorname{Feas}}

\newtheorem{definition}{Definition}
\newcounter{para}

\date{}

\begin{document}

\maketitle

\begin{abstract}
 In this paper, we revisit implicit regularization from the ground up using notions from dynamical systems and invariant subspaces of Morse functions. The key contributions are a new criterion for implicit regularization---a leading contender to explain the generalization power of deep models such as neural networks---and  a  general blueprint to study it.  We apply these techniques  to settle the conjecture on implicit regularization in  matrix factorization raised in~\cite{implicitregmatr2017}.

\end{abstract}

\section{Introduction}
Deep models, such as deep neural networks, have seen a tremendous growth in their range of applications, growth that far outpaced our theoretical understanding of them. One of the major outstanding questions is to understand their generalizing power: why a deep model fitted with a relatively small amount of data provides good performance for data points well outside its training set? On the one hand, fitting parameters uniquely to training data is very likely to not generalize well, on the other hand, having an underdetermined model leaves open the question of how to select among the many candidate parameters that fit the data. The training stage of a model can be cast as an optimization procedure in which a cost function  is minimized. This cost function measures a chosen notion of distortion between model parameters and training data.  For underdetermined models,  the cost function has a large number of minima; in fact, very often a continuum of them. 

When optimizing a cost function with many minima, the optimization method used dictates which minimum is selected.  This is in stark contrast with, say, a typical  convex optimization problem, where there is no uncertainty due to multiple minima and the effect of the optimization method is confined to the speed of convergence. This non-uniqueness of solutions is in many applications seen as little more than an inconvenience, and when one needs to obtain a unique solution, regularization methods, such as Tykhonov regularization, are nowadays well-understood. The issue at hand here is that the methods yielding the best generalizing power for deep models do not have {\it explicit} regularization terms, but are often simple gradient methods, thus offering no understanding of what makes a good set of parameters for the purpose of generalization. The theory of implicit regularization aims to uncover the rules of selection of parameters that are hidden in the training of deep models. It does so via the introduction of an auxiliary optimization problem,  whose solution should be essentially unique and coincide with the local minimum selected~\cite{neyshabur2014search,neyshabur2017geometry,soudry2018implicit}. This auxiliary problem is thought of as {\it regularizing} the original problem {\it implicitly}.

\paragraph{The conjecture and our approach} In order to better understand implicit regularization, the authors of~\cite{implicitregmatr2017} put forward  a remarkable gradient flow: the function minimized is \begin{equation}\label{eq:defJ1}J(X)= \sum_{i=1}^q (\tr(A_iX)-y_i)^2,\end{equation} where $X$ and $A_i$'s are  positive semi-definite matrices, and $y_i$ positive numbers\footnote{The additional assumption that the $A_i$'s are positive semi-definite and $y_i >0$ allows us to simplify the proof by making sure that all optimization problems are feasible. This assumption can be relaxed at the expanse of additional cases to treat in the proofs, which we omit here.}. The authors then conjecture that a gradient-like flow of $J$, when initialized near zero and in the underdetermined regime (i.e., $q\ll n^2$, where $n$ is the dimension of $X$, we take here the definition to mean $q\leq n$.),  converges arbitrarily close to  a global minimizer of the problem $$\min \|X\|_* \quad \mbox{ s.t. } \tr(A_iX)=y_i,$$ where $\|\cdot\|_*$ is the nuclear norm. The {\it explicitly} regularized problem is thus $$\min J(X)+ \lambda \|X\|_*.$$

The best known results in the literature about the conjecture on implicit regularization in matrix factorization are the ones in~\cite{implicitregmatr2017}: they showed it held for (i) the case $q=1$ and (ii) the case of $q \geq 1$ {\it commuting matrices} $A_i$, i.e., matrices so that $A_iA_j-A_jA_i = 0$, with initial state $X_0=I$, $I$ being the identity matrix. Their analysis was based on finding explicit forms of the solution of the gradient flow ODE for $X$. 
The approach we take in this paper is entirely different. However, we show in Example~\ref{ex:comm} below how their analysis of commuting matrices fits within our framework, and as a consequence, we provide a rather unexpected characterization of the convergence point in this case (see below Example~\ref{ex:comm}).

Our analysis leads us to believe that the conjecture of~\cite{implicitregmatr2017} is {\it ``mostly true''}. We mean by this  that while the conjecture is not true in its strictest form, a  possibility raised by the authors of the original paper in fact, and further substantiated in~\cite{arora2019implicit}, there are regimes in which it is provably true (the {\it tame spectrum} regime, defined herein), and when moving away from this regime, the performance appears to degrade only slowly. Using the results of our analysis, we  can also manufacture settings in which the prediction of the conjecture does not hold even approximately; but for ``typical'' data, it appears to hold. Interestingly, we relate the  performance  to the {\it spectral gap} of $\sum_{i=1}^q A_i$ and verify that for moderate to large spectral gaps, implicit regularization occurs. 

In the proof, we make use of elementary notions from Morse theory and dynamical systems, and we refer the reader to~\cite{milnor2016morse, banyaga2013lectures} for thorough introductions, as we keep the review of known material to a minimum in this paper. We summarize the proof of the conjecture below, after having introduced our general blueprint. We then conclude and provide numerical evidence supporting our conclusion. The entirety of the proof of the conjecture is relegated to the Appendix.

\section{Implicit Regularization: towards a general theory}\label{sec:impltheory}
Implicit regularization is in essence a notion of {\it compatibility} between two optimization problems. We propose here a way to quantify and understand this compatibility.

\subsection{Primal and regularization problems} The first of the two problems is  what we term the {\bf primal or training}  problem, it is given by \begin{equation}\label{eq:primalprob}\min J(\mu;x),\quad \mbox{ via } \dot x = f(\mu,x),\  x(0) = x_0(\mu)\end{equation} where $\mu$ represent parameters or data,  $x$ is the variable we are optimizing over (the parameters of the model to be fitted), and $\dot x = f(\mu; x)$ the method used to optimize $J$, e.g. $f(x) = -\frac{\partial J}{\partial x}$. 
We assume that $J(\mu;x)$ is at least $C^2$ in $x$. (In fact, $J$ is {\it real analytic} for the case of implicit regularization in matrix factorization). A presentation of the primal problem should always include a description of the optimization method used (here, $f(\mu;x)$) and initial state $x_0(\mu)$.

We denote by $\varphi_t(x_0)$ the solution at time $t$ of the ODE in~\eqref{eq:primalprob} with $x(0)=x_0$, and by  $\crit J$ the set of {\bf critical points} of $J(\mu,x)$, that is, the set of zeros of $\frac{\partial J}{\partial x}$. This set of course depends on $\mu$, but we often omit the explicit dependence to keep the notation simple. We denote by $\crit_0 J$ the set of {\it minima} of $J$, and refer to the locally stable zeros of $f$ as {\bf sinks}. When using the {\bf gradient flow} $f(\mu;x)= - \grad J$, the sinks are the local minima of $J$. The cases of interest are the ones where $\crit_0 J$ has large cardinality, and even contains  connected components.\footnote{We note that if $f$ is known, the function $J$ is in fact not required. The set of critical points of $J$ can be replaced by the sets of zero of $f$, and minima by locally stable zeros of $f$, etc.}   Unless $\crit_0 J$ is a singleton, the local minima to which~\eqref{eq:primalprob} will converge  depends on the initial state and the optimization method chosen.\footnote{We consider below gradient with respect to a metric defined by the data $\mu$. Hence $\grad J$ is not necessarily equal to $\frac{\partial J}{\partial x}$}

The second optimization problem, which  we term the  {\bf regularization problem}, describes to {\it which element} of $\crit_0 J$ the primal problem converges. It is a problem of the form  \begin{equation}\label{eq:defoS}\oR:\ \min K(\mu;x)\quad  \mbox{s.t. } r(\mu;x) = 0,\end{equation} where $K$ is a differentiable real-valued function and $r$ an $\R^q$-valued function. Both $K$ and $r$ can depend on $\mu$. The {\bf feasible set} of $\oR$, denoted by $\feas(\oR)$, is the zero  set of $r$ . It is required to be included in the set of minima of $J$: \begin{equation}\label{eq:deffeas}\feas(\oR):=\{x \mid r(x) = 0 \}\subseteq \crit_0 J.\end{equation} Said otherwise: {\it feasible points of $\oR$ are minima of $J$}. Since the trajectories of a gradient flow generically converge to a point in $\crit_0 J$, this requirement simply ensures that the regularization problem selects one such point. 

 \subsection{Pre-critical sets  and compatible problems} We  say that the primal problem is (exactly) {\bf implicitly regularized} by the regularization problem if for all $\mu$, $$\varphi_\infty(\mu;x_0) \in \arg \min K(\mu;x) \quad \mbox{ s.t. } r(\mu;x) = 0.$$ If $\varphi_\infty(\mu; x_0)$ is  approximately a minimizer of $K$, we refer to the above as {\bf approximately implicitly regularized}. 
 
 It is important to note that in the regularization problem, the optimization method is irrelevant, as we simply seek  global minima.  Many variations on this definition are possible, such as requiring that it holds for all $x_0$, or that we converge to any minima of $K(x)$, not necessarily a global one.

The following space, which we call the {\bf pre-critical set} of $\oR$ will play an essential role in implicit regularization: \begin{equation}\label{eq:defcritstar} \crit^*_\mu \oR:=\{x \in \R^n\mid \frac{\partial K(\mu;x)}{\partial x}+ \lambda^\top  \frac{\partial r(\mu;x)}{\partial x} =0 \mbox{ for some } \lambda \in \R^q\}.\end{equation} We can also define $\crit^*_\mu \oR$ as  the projection onto the $x$-coordinates of the set $A:= \{ (x,\lambda) \in \R^{n \times q} \mid \frac{\partial K(\mu;x)}{\partial x}+ \lambda^\top  \frac{\partial r(\mu;x)}{\partial x} =0\}.$ 
This is a space of $x \in \R^n$ which can be critical points of  the Lagrangian $L(\mu; x,\lambda) = K(\mu;x) + \lambda^\top r(\mu;x)$ of problem $\oR$.  More abstractly, we can think of it as the graph, over $\R^q$, of the implicitly defined function $x(\lambda)$ given by $\frac{\partial K}{\partial x}+ \lambda^\top  \frac{\partial r}{\partial x} =0$. This space can be fairly complex: multi-valued, containing several connected components and non-smooth points, as we will observe in the examples below. The terminology pre-critical set comes from the fact that points in $\crit^* \oR$ which are in $\feas(\oR)$ are critical points of $\oR$: $\crit \oR = \feas(\oR) \cap \crit^* \oR$.

The space $\crit^* \oR$ has a natural role in implicit regularization:   if $\crit^* \oR$ is an {\it invariant subspace}\footnote{Recall that $S$ is an invariant subspace for the dynamics $\dot x = f(x)$ if for $x_0 \in S$, the solution  $x(t) \in S$ for all $t$.} for the primal dynamics, then implicit regularization is in a sense more likely to occur: indeed, if we initialize the primal flow in this space, or if this space is an attractor for the primal flow, the primal flow converges to a critical point of the Lagrangian of $\oR$; said otherwise, it shows that {\it the regularization problem is  well-matched  to the primal problem}.  Furthermore, since we are interested in minimizers of $\oR$, it is sufficient to consider components of $\crit^* \oR$ containing the minimizers of $K$, we call such a set of these components  $\crit^*_{\mu,0} \oR$ or simply $\crit_{0}^* \oR$. Denote by $M$ a set of data point $\mu$. We have the following definition:

\begin{definition}[Compatible primal and regularization problems]\label{def:compat}
We say that the primal problem~\eqref{eq:primalprob}  and the regularization problem  $\oR$ of Eq.~\eqref{eq:defoS} are {\bf compatible} over $M$ if the space $\crit^*_{\mu,0} \oR$ is invariant for $f(\mu;x)$, for all $\mu \in M$.
\end{definition}
We illustrate the definition on four examples. In particular, we revisit the approach of~\cite{implicitregmatr2017} on implicit regularization on matrix factorization and show that Definition~\ref{def:compat} yields new insights to it.

\begin{exampleA}[Trivial regularization problem]\label{ex1}\normalfont
Starting from the primal problem, it is always possible  to construct a regularization problem. Perhaps the simplest one, which we term the {\bf trivial} regularization problem, is given by  $$K(\mu;x) = \|\varphi_\infty(\mu,x_0)-x\|^2.$$ We call it trivial since the set $\crit^*_\mu \oR$ for this regularization problem consists of the singleton $\{\varphi_\infty(\mu;x_0)\}$. The regularization problem and primal problem are thus compatible in the sense of Def.~\ref{def:compat}, since $\varphi_\infty(x_0)$ is an equilibrium point of Eq.~\eqref{eq:primalprob}, and thus an invariant set for the dynamics. $\qed$
\end{exampleA}
\begin{exampleA}\label{ex2}\normalfont Let $x, \mu_i  \in \R^n$, and $y_i \in \R$, for  $1 \leq i \leq q$.  Set $$J(\mu;x)= \sum_{i=1}^q (\mu_i^\top x - y_i)^2.$$ Using the gradient for the Euclidean inner product as optimization method, and setting $x_0(\mu)=0$, we have \begin{equation}\label{ex:srik}\dot x = -\sum_{i=1}^q (\mu_i^\top x - y_i)\mu_i, \quad x_0 = 0.\end{equation} This problem is implicit regularized by the regularization problem $$\oR: \min \frac{1}{2}\|x\|^2 \quad \mbox{ s.t. } \mu_i^\top x - y_i = 0.$$ To see this, it suffices to solve the linear differential equation~\eqref{ex:srik}. We obtain in this case that $$\crit^* \oR = \{ x \in \R^n \mid x+ \sum_{i=1}^q \lambda_i \mu_i =0 \mbox{ for some } \lambda \in \R^q\},$$ equivalently, $\crit^* \oR = \Span \{\mu _i\}$. This set is clearly invariant for the dynamics~\eqref{ex:srik} and thus the problems are compatible  in the sense of Definition~\ref{def:compat}, with $M = \R^n$.
\qed	
\end{exampleA}
In the following example, the set $\crit^* \oR$ has a richer structure than in the previous two examples.
\begin{exampleA}[Matrix factorization with $q=1$] 
\normalfont Let $A \in \R^{n \times n}$ be a positive definite matrix with distinct eigenvalues, $U \in \R^{n \times n}$ and $y$ be a positive number. Consider the primal problem $$J(U)= (\tr(A UU^\top) -y)^2, \quad \frac{d}{dt} U = -(\tr(AUU^\top)-y) A U, U(0)=U_0,$$ and the regularization problem
$$ \min \tr UU^\top \quad \mbox{ s.t. } \tr(AUU^\top) = y.$$
Writing $X=UU^\top$, it is easy to see that it is the problem of implicit regularization in matrix factorization. A short calculation, which uses the fact that $A$ is symmetric, yields $$\crit^* \oR = \{ U \in \R^{n \times n} \mid (I-\lambda A)U=0\mbox{ for some }\lambda \in \R\}.$$ The equation $(I-\lambda A)U$ admits non-trivial solutions only for $\lambda \in \operatorname{spec}(A)$, where we denote by $\operatorname{spec}(A)$ the set of eigenvalues of $A$, and in this case $U=vw^\top$, where $v$ is an  eigenvector corresponding to $\lambda$, and $w\in \R^n$ is arbitrary. Hence $$\crit^* \oR = \cup_{\lambda_i \in \operatorname{spec}(A)} \{v_i w^\top \mid w \in \R^n, v_i \mbox{ an eigenvector corresponding to }\lambda_i\},$$ this set is the union of $n$-dimensional subspaces of $\R^{n \times n}$. Now recalling that eigenvectors of $A$ are orthogonal, it is easy to verify that each branch of $\crit^*\oR$ is invariant for the dynamics of $U$ given above, thus showing compatibility of the primal and selection problems. \qed
 \end{exampleA}
 
 The final example addresses the case studied in~\cite{implicitregmatr2017}.
 \begin{exampleA}[Matrix factorization with commuting matrices]\label{ex:comm}
 \normalfont
 Now assume we have $q$ positive definite matrices $A_i$,  with primal problem 
 $$J(U)= \sum_{i=1}^q(\tr(A_i UU^\top) -y_i)^2, \quad \frac{d}{dt} U = -\sum_{i=1}^q(\tr(A_iUU^\top)-y_i) A_i U, U(0)=U_0$$ and regularization problem
 $$ \min \tr UU^\top \quad \mbox{ s.t. } \tr(A_iUU^\top) = y_i, i=1,\ldots, q$$

The set pre-critical set is 
 $$\crit^* \oR = \{ U \in \R^{n \times n} \mid (I-\sum_{i=1}^q\lambda_i A_i)U=0\mbox{ for some }\lambda \in \R\}.$$
This set is in general difficult to study, as it requires to determine when the affine space of matrices $I+\operatorname{span}\{A_i\}$ contains rank deficient matrices. If we assume that the $A_i$ commute, however, the situation is far simpler: commuting symmetric matrices admit a common set of eigenvectors, hence there exists an orthogonal matrix $V$ such that $ A_i =VD_iV^\top$, where $D_i$ is a diagonal matrix, with diagonal entries $d_{ij}$, $i=1,\ldots,q, j=1,\ldots, n$, and the columns of $V$ are eigenvectors of the $A_i$. We denote these columns by $v_i$.

Now let $I \subseteq \{1,\ldots,n\}$ be a subset of cardinality $q$. We can, generically for the $d_i$'s, find a unique solution $\lambda=(\lambda_1,\ldots, \lambda_q)$ to the linear system
 $$\sum_{i=1}^q \lambda_i d_{i,k} =1, k \in I. $$ With these $\lambda_i$'s, a short calculation shows that the matrix $(I- \sum_{i=1}^q \lambda_i A_i)=V(I- \sum_{i=1}^q \lambda_i D_i)V^\top$ has a kernel of dimension $q$, spanned by the vectors $v_i, i \in I$. 
 There are ${n \choose q}$ such subsets $I$, and to each of them corresponds a vector $\lambda \in \R^q$ and thus a component of dimension $q\times n$ of  $\crit^* \oR$:
  $$\{U \in \R^{n \times n} \mid U = v w^\top, w \in \R^n, v \in \operatorname{span}\{v_i \mid i \in I\} \}.$$ Similarly as in the previous example, this set is easily seen to be invariant for the primal dynamics.  
  
  The analysis suggests that when the  flow converges to a rank one matrix (which we show below is the case) $$X=vv^\top,$$ then there exists $q$ eigenvectors in the common set of eigenvectors of the $A_i$ so that $v$ is a linear combination of these $q$ eigenvectors. Said otherwise, there exists a {\it sparse vector} $x \in \R^n$, with $\|x\|_0 \leq q$ so that $$v = Vx.$$   
   This is a non-trivial statement  when $q <n$, and we verified it in simulations. Note that there are ${n \choose q}$ such linear subspaces, and when $q$ is large, they can intersect non-trivially.
     \qed
\end{exampleA}

\subsection{A blueprint for implicit regularization} The criterion of Definition~\ref{def:compat} by itself is clearly not sufficient for implicit regularization to occur. For example, since the space $\crit_0 J$ contains points $x$ not in $\feas(\oR)$, one can converge to a non-feasible point for $\oR$, even when initialized in $\crit^* \oR$. More, even if  $\feas(\oR) = \crit_0 J$, saddle points of the primal dynamics can become local sinks in an invariant space, allowing for the primal flow to converge to a non-feasible point. 
 This criterion, however, provides us with a general blueprint to study implicit regularization:
\begin{enumerate}
\item 	Identify a set of $\mu$'s for which primal and regularization problems are compatible. Denote this set by $M$. Of course, the larger the set, the better. 
\item Verify that when initialized in the pre-critical set $\crit^*_0 \oR$ with $\mu \in M$, the primal dynamics will converge to a feasible point: $$\varphi_{\infty}(x_0) \in \feas(\oR) \quad\mbox{generically for } x_0 \in \crit^* \oR$$

 This entails verifying that no saddles of $J$ are local sinks for the  dynamics in $\crit^* \oR$. 
\item Verify that under the above conditions, the primal dynamics in fact converges to a minimum of $K$ (the first two points guarantee it converges to a critical point).
\end{enumerate}
If these three items are met, a form of implicit regularization holds: for  initial conditions $x_0(\mu) \in \crit_0^* \oR$, and $\mu \in M$, implicit regularization is generically true. If only the first two points are verified, a weaker form of implicit regularization, whereby the primal flow converges to a critical point of the regularization problem, and not necessarily a global minimum, holds.

Depending on the flavor of implicit regularization one is after, additional steps can be pursued. We consider here the following two:

\begin{enumerate}
\item[4.] If $x_0$ is independent of $\mu$, verify that the primal dynamics converges to $\crit^* \oR$ from $x_0$.
\item[5.] Verify that when $\mu \notin M$, the performance does not degrade drastically (i.e., the primal converge to an approximate minimizer of the regularization problem.
\end{enumerate}
The last item is in general difficult to verify rigorously, but one can appeal to continuity to obtain qualitative statements. 
We also note that the items are not completely independent from each other. For example, if some local sinks in the invariant space $\crit^*_0 \oR$ are saddles of the general dynamics, flow lines near that saddle will escape the vicinity of $\crit^*_0 \oR$, and point 4 is less likely to hold. We will revisit this point in detail later in the paper.

A key point of the blueprint is to identify an as large as possible set $M$ of a data points for which implicit regularization is exact. For the first three examples given above, $M$ was equal to the entire space of possible $\mu$'s. For the last examples, it was the set of commuting matrices. Another aspect is the dimension of the pre-critical set: in the first example, its dimension  was zero, whereas in the second, it was $q$. Adopting a more subjective point of view, we can say that the regularization described in Example~\ref{ex1} is not as useful or surprising as the one of Example~\ref{ex2}. A useful heuristic is that the {\it larger the dimension} of the pre-critical set $\crit^* \oR$, the more useful the implicit regularization is.
Finally, we mention that not all sets $M$ are equivalent for  implicit regularization insofar as item $5$ is concerned. For example, the set exhibited in Example~\ref{ex:comm}, which originated from the paper~\cite{implicitregmatr2017}, does not lend itself well to generalization. This was in fact pointed out by the authors of the paper, though they arrived at this conclusion from a very different perspective, namely looking for {\it explicit solutions} of the primal problem. We reach this conclusion by noting that if we perturb one matrix $A_i$ by a small amount, the trajectory of the primal system can stray very far from $\crit^*\oR$. Our proof below will exhibit a different set $M$, which we refer to as matrices with {\it tame spectrum}\footnote{We say that positive semi-definite  matrices $A_i$, $i=1,\ldots, q$ have a tame spectrum if the eigenvalues of $\sum_{i=1}^q A_i)$ are given by $\{\alpha, \beta,\cdots, \beta, 0,\ldots, 0\}$ for some $\alpha>\beta >0$. See Def.~\ref{def:LE}.}  , that is better suited to study implicit regularization in matrix factorization.

\subsection{How to determine the implicit regularizer?}

A fundamental goal in the area is to {\it determine} the implicit regularizer $\oR$ given a primal flow. We provide here a brief overview of how our blueprint provides a path to obtain such regularizer, but we postpone a thorough analysis to a forthcoming publication.

We consider the training problem with cost $J = \sum_{i=1}^q l(\mu_i,x)$, where $l$, the {\it loss function}, depends on data points $\mu_i$ and parameters $x$. We assume that $l$ is positive semi-definite with minimum at zero,  hence $\min J =0$. In Example~\ref{ex2} , $l(\mu,x)= (c^\top x - y)^2$ (and $\mu = (c,y)$), and in Example~\ref{ex:comm}, $l(\mu,x) = \tr(AX)-y$, with $\mu = (A,y)$. The primal flow is taken to be the natural gradient of $J$: \begin{equation}\label{eq:defgradJloss}\dot x = -\sum_{i=1}^q \frac{\partial \l}{\partial x}(\mu_i,x)=f(\mu,x). \end{equation}

Following our Ansatz, we need to find a function $K(x)$ so that the  regularization problem $\oR: \min K(x) \quad \mbox{ s.t. } l(\mu_i,x)=0$ is compatible with the above primal flow; said otherwise, so that  $\crit^* \oR$ is invariant for~\eqref{eq:defgradJloss}. We can show, but we omit the derivation here, that this requirement of invariance reduces to the system of partial differential equations (with unknowns $K$ and $\lambda_i$)

\begin{equation}\label{eq:mastereq}
\frac{\partial^2 K}{\partial x^2}f(\mu,x) + \sum_{i=1}^q\lambda_i  \frac{\partial^2 l}{\partial x^2}(\mu_i,x) f(\mu,x)=0.
\end{equation} This partial differential relation has a simple interpretation: the Hessian of $K$, acting on $f$, is a linear combination of the Hessians of $l$ at the datapoints $\mu_i$, acting on $f$. 

 As an example of how one can use this equation to determine potential implicit regularizers, consider the case of a loss function \begin{equation}\label{eq:exloss}l(\mu,x) = \sigma(c^\top x - y),\end{equation} where $\sigma$ is a twice differentiable real-valued (nonlinear) function with minimal value zero, $c \in \R^n$ and $y \in \R$. Its derivative and Hessian are given by \begin{equation} \label{eq:derivativessigm}
 \frac{\partial l}{\partial x} = \sigma'(c^\top x -y)c \quad \mbox{ and }\quad  \frac{\partial^2 l}{\partial x^2} = \sigma''(c^\top x-y)c c^\top,
 \end{equation} where $\sigma'$ and $\sigma''$ are the first and second derivatives of $\sigma$, respectively. We use the shorthand $\sigma_i:=\sigma(\mu_i,x)$.
 One can then use Eq.~\eqref{eq:mastereq} to produce candidate regularization problems. For the case described in Eq.~\eqref{eq:exloss}, a short calculation easily produces two such candidates. Using the specific form of $f$ given in~\eqref{eq:defgradJloss}, we find that taking $\partial^2 K/\partial x^2$ to be a multiple of the identity matrix is a solution of~\eqref{eq:mastereq} for  $\lambda_i= d\log(\sigma'_i)(\sum_j c_i^\top c_j \sigma'_j)^{-1}$ (here, $d\log f = d/dx(\log(f))$). Clearly, taking $K=\|x\|^2$ fits the requirement and thus yields a regularization problem compatible with the primal problem. Another candidate is $K=x^\top Q x$ with $Q = \sum_{i=1}^q c_ic_i^\top$, and $\lambda_i= (\sigma''_i)^{-1}$, but in this case the regularizer depends on the data. More can be extracted from Eq.~\eqref{eq:mastereq}, but this is outside the scope of this paper.

\paragraph{Implicit regularization in matrix factorization}

We now focus the above discussion to implicit regularization for matrix factorization, and describe the contents of the remainder of the paper in more details. In a nutshell, we will illustrate how to apply items 1-5 of the above blueprint to the conjecture proposed in~\cite{implicitregmatr2017}.

The primal problem is given by the differential equation \begin{equation}\label{eq:flow1}\frac{d}{dt} X = -\sum_{i=1}^q (\tr(A_iX) - y_i) (A_i X+XA_i) ,\quad X(0)= X_0\delta\end{equation} where $X_0 \in \R^{n \times n}$ is a positive semi-definite matrix, $\delta$ is small constant (e.g. $10^7$ times smaller than the $y_i$'s, $\|A_i\|$'s and $\|X_0\|$), the $A_i \in \R^{n \times n}$ are positive semi-definite and $y_i>0$, for $1 \leq i \leq q$. We will describe below the function $J$ this flow minimizes. This type of matrix differential equation has a long history. Related flows were shown by Brockett in~\cite{brockett1991dynamical}  to solve a variety of combinatorial problems, and the monograph~\cite{helmke2012optimization} provides an in-depth look at many of their characteristics. For example, it is easy to see that the flow of~\eqref{eq:flow1} preserves the cone of positive semi-definite matrices and moreover if $\operatorname{rank}(X_0) = k$, then $\operatorname{rank} (X(t)) \leq k$ for $t \in [0,\infty]$~\cite{helmke2012optimization}. 
 
 It was observed that when initialized near zero, i.e., when $\delta$ is small, the flow of~\eqref{eq:flow1} converged to (near) an $X^*$ with the property \begin{equation}\label{eq:sel1} \oR: X^* \in \arg \min_{X} \|X\|_* \quad \mbox{ s.t. } \tr (A_iX) = y_i, 1\leq i \leq q,\end{equation} where $\|X\|_*=\tr X$ is the {\it nuclear} or {\it trace norm} of $X$. With our terminology, it was conjectured, roughly speaking, that the regularization problem $\oR$ of Eq.~\eqref{eq:sel1} approximately regularizes the flow of Eq.~\eqref{eq:flow1} in the limit $\delta \to 0$. While convergence of the flow to the set of matrices that meet the constraints $\tr(A_iX)=y_i$ may not appear surprising given the form of~\eqref{eq:flow1}, convergence to a {\it global }minimum of $\oR$ was certainly unexpected. 
As already mentioned, we believe that exact regularization as in Example~\ref{ex2} does not take place here, but yet via exhibiting a set of $\mu$'s for which it does, one can expect implicit regularization to be approximately true (as in step 5 of the blueprint).
\begin{remark}\label{rem:relax}
	We consider below (see Eq.~\eqref{eq:optimprobgen})  the family of problems $$\oR_k: \min \| X\|_* \quad \mbox{ s.t. }\quad  \tr (A_iX)=y_i, \operatorname{rank} X=k,$$ parametrized by the rank $k$ of $X$, $k=1,\ldots,n$. We then, in essence, show that under certain assumptions, the solution of the rank constrained problem with $k=1$ and the problem with $k=n$ agree (see Theorem~\ref{th:equivnfullselecmin}). Said otherwise, it says that under these assumptions, the convex relaxation (problem with $k=n$) of the rank constrained problem ($k=1$) is exact. It is in a sense remarkable that the conditions guaranteeing compatibility of primal and regularization problem also yield exact relaxation.
\end{remark}

We provide below a summary of the steps taken in the proof of the conjecture: 
\begin{enumerate}
\item We show that the flow of~\eqref{eq:flow1} {\it always} goes near a rank $1$ matrix, denote it by $X_1$. {\it (Theorem~\ref{th:rk1bottle})}
\item We introduce the so-called {\it normal form} for the system, and show that it  exists  generically in the underdetermined regime (termed {\it rank spread condition}).  The normal form makes much of the proof more transparent. ({\it Lem.~\ref{lem:P} and Prop.~\ref{prop:optnormcoord}})

The construction of the normal form itself can be omitted at first reading, and one can immediately study the normal system described in Eq.~\eqref{eq:nomsys}, and the corresponding regularization problem~\eqref{eq:optx}.

\item We show that the function $J$ of Eq.~\eqref{eq:defJ1} is a Morse-Bott function and the primal flow is a gradient of this function for a metric defined by the $A_i$'s. ({\it Th.~\ref{th:morsebott}})
\item We show that the conjecture holds for commuting generalized  projection matrices. This part is meant to illustrate the use of the items above by showing how they immediately provide extensions on the extant result in the area, and can be skipped at first reading. ({\it Prop.~\ref{prop:commproj} and Cor.~\ref{cor:commproj}}) 
 
\item We introduce the {\it tame spectrum condition}. It describes a set $M$ of parameters for which the pre-critical set  $\crit_0^*\oR$ is invariant for the dynamics $\rightarrow$ the problems are compatible over tame matrices. ({\it Def.~\ref{def:LE}, Th.~\ref{th:compat}})

\item We show that $\crit_0^*\oR$ and $\crit_0 J$ intersect transversally and that no saddles of $J$ are sinks in $\crit_0^* \oR$ $\rightarrow$ when $X_0 \in \crit_0^* \oR$, the flow converges to a feasible point of $\oR$. ({\it Th.~\ref{th:rank1}}).
\item We show that $X_1$ of item 1 {\it belongs} to $\crit_0^* \oR$ $\rightarrow$ when initialized near zero, the flow always goes near $\crit_0^* \oR$. ({\it Prop.~\ref{prop:winLv}})
\item We show, relying on Lojasiewicz inequality, that when initialized near $\crit_0^* \oR$, the flow converges to a point near $\crit_0^* \oR$. ({\it  Prop.~\ref{prop:closeLoja}}).
\end{enumerate}
The above items roughly cover points $1,2$ and $4$ of the blueprint. The analysis for point $3$ is done in the last part:

\begin{itemize}
\item[9.] Introduce {\it intrinsic} coordinates on $\crit^*_0 \oR$. Write the dynamics and regularization problem in these coordinates (we call them reduced coordinates). ({\it Lem.~\ref{lem:reddyn}}).
\item[10.] Show that in  $\crit_0^* \oR$, $J$ becomes a {\it Morse} function, that it has $3^q$ critical points, of which $2^q$ are minima, and only $2$ of these correspond to global minima of $\oR$. ({\it Th.~\ref{th:reddyngradient}}).
\item[11.] Show that when initialized at the $X_1$ of item $1$, the flow converges to one of the two {\it global} minima of $\oR$ (only sketch the last part) ({\it Cor.~\ref{cor:lambdasign}}).
\item[12.] Show that under the tame spectrum hypothesis, there always exists a global minimum of $\oR$ of rank $1$. ({\it Th.~\ref{th:equivnfullselecmin}}) $\rightarrow$ studying the problem in $\crit^*_0 \oR$ can be done without loss of generality.
\end{itemize}
\section{Background and notation}

\paragraph{Problem set-up} Let $n,q$ be positive integers and let $A_i \in \R^{n \times n}$   be real symmetric positive semi-definite (psd) matrices  of rank $r_i$ and $y_i$ be positive numbers,  for $i=1, \ldots, q$. We denote by $S_{k,n}$ the space of psd matrices in $\R^{n \times n}$ of rank at most $k$ and write $S_n$ for $S_{n,n}$. The primal (training) problem is  the Cauchy problem

\begin{equation}\label{eq:mainsys}
\dot X(t) = - \sum_{i=1}^q \left(\tr(A_iX(t))-y_i\right) \left(A_iX(t)+X(t)A_i\right), \quad X(0)= X_0\delta
\end{equation}
where $X_0$ is a real symmetric matrix and $\delta>0$ a constant. We observe that $\dot X = \dot X^\top$, i.e., symmetric matrices are an invariant set of system~\eqref{eq:mainsys}, and hence $X(t)$ is symmetric for all $t > 0$ for which the solution exists.\footnote{We show below that a simple Lyapunov argument establishes existence of solutions for $t>0$, and will thus omit this qualifier from now on.}  In fact, more is true:   system~\eqref{eq:mainsys} leaves the cone of positive semidefinite matrices invariant and does not increase rank as mentioned earlier. Hence, if $X_0$ is positive semidefinite of rank $k$, then $X(t)$ is also psd and of rank at most $k$.

This motivates the introduction of the flow 
\begin{equation}\label{eq:mainUsys}
    \dot U(t) = - \sum_{i=1}^q\left(\tr(A_iUU^\top)-y_i\right) A_i U(t),\quad  U(0)=U_0,
\end{equation} 
where $U \in \R^{n \times k}$, whose trajectories can be mapped onto trajectories of~\eqref{eq:mainsys} as follows:

\begin{lemma}\label{lem:equivXU} Let $U(t)$ be the solution of~\eqref{eq:mainUsys} with $U_0 \in \R^{n \times k}$, and let $X(t)$ be the solution of~\eqref{eq:mainsys} with $X_0=U_0U_0^\top$, then $U(t)=X(t)$ when they exist.
\end{lemma}
To verify that  the Lemma holds, it suffices to differentiate $\bar X(t):=U(t)U(t)^\top$ and observe that $\bar X$ and $X$ then obey the same Cauchy problem. If $X_0$ is of rank $k$, we note that there exists an $O(k)$-parametrized family of $U_0$ so that $X_0=U_0U_0^\top$, where $O(k)$ is the orthogonal group in dimension $r$.

From now on, we use the notation \begin{equation}\label{eq:defrho} \rho_i:=\tr(A_iUU^\top)-y_i,
 \end{equation} or, depending on the context, $\rho_i:=\tr(A_iX)-y_i.$

Consider the real-valued function 
\begin{equation}\label{eq:maincostU}
J(U) := \frac{1}{4}\sum_{i=1}^q(\tr(A_iUU^\top)-y_i)^2.
\end{equation}
A short calculation shows that the flow~\eqref{eq:mainUsys} is  the {\it gradient flow} of~\eqref{eq:maincostU} for the Euclidean inner product on $\R^{n \times n}$: 
$$\grad J(U) = \sum_{i=1}^q\rho_i A_i U(t).$$
From Lemma~\ref{lem:equivXU}, we conclude that the function
\begin{equation}\label{eq:maincost}
J(X) = \frac{1}{2}\sum_{i=1}^q\left(\tr(A_iX)-y_i\right)^2
\end{equation}
is a Lyapunov function for the flow~\eqref{eq:mainsys}. We also say that~\eqref{eq:mainsys} is {\it gradient-like} for $J$. (Note that we overload the notation for $J$ as well; the context will dispel possible confusions). This fact can be used to show existence of solutions of Eq.~\eqref{eq:mainsys}.

We now introduce a slight generalization of the {\bf regularization problem} for implicit regularization, allowing for a rank $1 \leq k \leq n$ for $X$ (and $U$)
\begin{equation}
\label{eq:optimprobgen}\oR_k:  \min_{X \in S_{n,k}} \tr X \quad \mbox{ s.t. } \tr (A_i X) = y_i, \quad i=1,\ldots,q
\end{equation}
and its equivalent in the $U$ coordinates
\begin{equation}
\label{eq:optimprobgenU} \oR_k:\min_{U \in \R^{n \times k}} \tr UU^\top  \quad \mbox{ s.t. } \tr (A_i UU^\top ) = y_i, \quad i=1,\ldots,q.
\end{equation}
The cases $k=1$ and $k=n$ will be of most interest to us.
\paragraph{Notation and conventions}\label{ssec:notation}
We gather here some of the notation used throughout the paper. We let $e_i$ be the vector in $\R^n$ with all entries equal to $0$, except for the $i$th one, which is equal to $1$. In general, the dimension of the matrices and vectors introduced in this section will depend on the context (e.g., $e_i$ could be a vector in $\R^m$ with $m < n$ as well). We generally use $c>0$ to denote positive constant. The value of $c$ can change throughout the argument without further comments. We denote by $I$ the identity matrix. When we need to emphasize the dimension, we write $I_n$ for the identity matrix in $\R^{n \times n}$. We let $O(n)$ be the orthogonal group: $\Theta \in O(n)$ if $\Theta \Theta^\top = I$.

Let $X$ be a positive semi-definite (psd) matrix. We say that $X$ is of {\bf $\epsilon$-rank $r$} if there exists a psd matrix $X_r \in S_{r,n}$ so that \begin{equation}\label{eq:ang}\|X-X_r\|/\|X_r\| \leq \epsilon.\end{equation} We will also informally say that $X$ is {\bf essentially of rank} $r$. For example, if $X$ is itself of rank $r$, then the previous inequality is trivially met. This type of bound is necessary to quantify when a matrix is close to a subset of low rank matrices, since the zero matrix is in the closure of the sets of matrices of rank $k$ for all $k$. Indeed, $I\delta$ is arbitrarily close, in the Euclidean distance, to the set of rank $1$ matrices but is of $\epsilon$-rank $1$ only for $\epsilon \geq 1$. A more geometric interpretation of~\eqref{eq:ang} is that the angle between the lines spanned by $X$ and $X_r$ is small when $\epsilon$ is small.   We say that $x$ is $\epsilon$-close to $y$ is $\|x-y\| \leq \epsilon$.

For a matrix $B \in \R^{n \times p}$, we denote by $\Span\{B\}$ the vector subspace of $\R^n$ spanned by the columns of the $B$.

Given a collection $N$ of disjoint subsets $$N_1,\ldots, N_q \subseteq \{1,\ldots,m\},$$  we denote by $\cD_N$ the vector space of diagonal matrices with entries $d_i$ satisfying $d_i = d_j$ if $i,j \in N_k$ for some $1 \leq k \leq q$. For example, if $N_1=\{1,2\}$ and $N_2=\{3 \}$ then matrices in $\cD_N$ are of the form $\diag([\alpha, \alpha, \beta])$, $\alpha, \beta \in \R$. Throughout the paper, we also use \begin{equation}\label{eq:def:Ei}
 E_i : = \sum_{j \in N_i} e_je_j^\top,	
\end{equation} the sets $N_i$ will be clear from the context.
 The $E_i$'s are thus diagonal matrices with $0$ on the diagonal, except in the entries indexed by $N_i$, which are 1. Continuing the previous example, we have $$E_1 = \diag(1,1,0) \mbox{  and  } E_2 = \diag(0,0,1).$$ We let $|N_i|$ be the cardinality of $N_i$. Given a vector $x \in \R^n$, we denote by $\bar x_i\in \R^{n}$ its canonical projection onto the column span of $E_i$. For the sets $N_i$ described above, we have $$\bar x_1 = \begin{pmatrix} x_1 & x_2 & 0\end{pmatrix}^\top \mbox{  and  }\bar x_2 =\begin{pmatrix}0 & 0 & x_3  \end{pmatrix}^\top.$$ Depending on the context, we may omit the zero entries and consider $\bar x_1 \in \R^2$ and $\bar x_2 \in \R$ or, more generally, $\bar x_i \in \R^{|N_i|}$. For a matrix $U \in \R^{n \times n}$, we denote by $(u)_{1..m,1..n}$ the submatrix of size $m \times n$ obtained by keeping the rows $1,\ldots, m$ and columns $1,\ldots, n$.

Given a vector subspace $L \subset \R^n$, we denote by $L^\perp$ its Euclidean orthogonal, i.e. $x \in L^\perp$ if $x^\top y =0$ for all $y \in L$.  Recall that $f(n)=\Theta(g(n))$ if there exists constants $c_1<c_2$ so that $c_1g(n) \leq f(n) \leq g(n)$ for all $n$ large enough. 
 
\section{The rank $1$ matrix bottleneck}\label{sec:bottleneck}

The rank-$1$ bottlneck property of the flow refers to the fact that when initialized near zero, without any additional assumptions, the flow of Eq.~\eqref{eq:mainsys} will be of $\varepsilon$-rank $1$ at some time $t_1$, for an arbitrarily small $\varepsilon$ provided that $X_0$ is small enough.

The proof of the rank $1$ bottleneck property contains two steps. In the first step, we exhibit a linear ODE and show that its solutions have the rank $1$ bottleneck property. In the second step, we show that the system of Eq.~\eqref{eq:mainUsys} follows the trajectory of this linear ODE  closely for a positive time $t_1$. The two steps together easily yield the proof for the general nonlinear system of Eq.~\eqref{eq:mainsys}.

\paragraph{Rank $1$ bottleneck for linear systems} The following result shows that the trajectories of linear differential equation $\dot V = A V$, $A$ a positive semi-definite matrix, which start at a small, non-zero initial  condition $V(0)$ of arbitrary rank, will eventually  be  of $\epsilon$-rank $1$ generically for $V(0)$. If $V(0)$ is also of rank $1$, then the statement is trivially true.

\begin{lemma}\label{lem:solinapprox}
Let $A \in \R^{n \times n}$ be a symmetric positive semi-definite matrix, $V_0 \in \R^{n \times n}$. Define $V(t,\delta):=e^{At}V_0\delta.$  Then, generically for $A$ and $V_0$, there exists a matrix-valued function $V_1(t,\delta)$ of rank $1$ such that for all $\epsilon>0$, there exists  $t^*>0$,  so that  $$\frac{\|V_1(t,\delta)-V(t,\delta)\|}{\|V_1(t,\delta)\|}\leq \epsilon, \mbox{ for all } t \geq t^*, \delta >0$$ Furthermore, $\|V_1(t,\delta)\| = \Theta(\delta)$ for all fixed $t>0$.\end{lemma}
\begin{proof}
Let $P \in O(n)$ be so that $PAP^\top=D$ where $D$ is a diagonal matrix with diagonal entries $d_1 \geq d_2 \geq \cdots \geq d_n\geq 0$ and  let $\bar V(t) := PV(t)P^\top$. Then $\bar V(t) = e^{Dt}\bar V_0$, where $\bar V_0 = PV_0P^\top\delta$. Because $e^{Dt}$ is diagonal, we have $$\bar V(t) = \sum_{i=1}^n e^{d_it} e_i \bar V_{0,i}\delta,$$ where $\bar V_{0,i}$ is the $i$th row of $\bar V_0$. 
Define \begin{equation}\label{eq:defV1lem}\bar V_1(t,\delta):= e^{d_1 t} e_1 \bar V_{0,1}\delta.\end{equation} Then $\|\bar V_1(t,\delta)\| = e^{d_1t} \|\bar V_{0,1}\|\delta$ and  $$\|\bar V(t,\delta)-\bar V_1(t,\delta)\| = \|\sum_{i=2}^n e^{d_i t}e_i\bar V_{0,i}\delta\| \leq c (n-1) e^{d_2t} \delta.$$ Normalizing by the norm of $\bar V_1(t)$, we have

$$\frac{\|\bar V(t,\delta)-\bar V_1(t,\delta)\|}{\|\bar V_1(t,\delta)\|} \leq \frac{(n-1) e^{d_2t}c \delta}{e^{d_1t} \|\bar V_{0,1}\|\delta} \leq c e^{(d_2-d_1)t}.$$  Generically for $A$, $d_2 - d_1 <0$, and thus taking $t^*$ large enough yields the first statement. The second statement is obvious from~\eqref{eq:defV1lem}.
\end{proof}

\paragraph{The error system} The following result gives conditions under which the trajectories of~\eqref{eq:mainUsys} are well-approximated by trajectories of $\dot V=AV$. Clearly, the approximation will be valid only for a bounded set $[0,T]$, as the solutions $V(t)$ generically diverge, whereas the solutions of~\eqref{eq:mainUsys}, being trajectories of the gradient flow of $J$,  are easily seen to be bounded.
\begin{lemma}\label{lem:errorU}
	Let $A_i$ be positive semi-definite matrices and $y_i >0$, $i=1\ldots q$, and set $$A:= \sum_{i=1}^qy_i A_i.$$  Let $U(t,\delta)$ be the solution of~\eqref{eq:mainUsys} with initial condition $U(0)=U_0\delta$, for $U_0 \in \R^{n\times n}$ nonzero, and let $V(t,\delta)$ be the solution of $\dot V = AV$, $V(0)=U(0)$.  Then, for all $0 < \epsilon <1$ and $t^*>0$, there is $\delta_1 >0$ so that $$\|U(t,\delta)-V(t,\delta)\| \leq \epsilon$$ for all $0 < \delta < \delta_1$, $0 \leq t \leq t^*$. Furthermore, $\|U(t,\delta)-V(t,\delta)\| = O(\delta^3)$ for $0 < \delta \leq \delta_1$, $0 \leq  t\leq t^*$.
\end{lemma}

\begin{proof}Using the matrix $A$ introduced in the Lemma's statement, we can rewrite~\eqref{eq:mainUsys} as $$\dot U = -\sum_{i=1}^q(\tr(A_i UU^\top)-y_i) A_i U= AU - \sum_{i=1}^q\tr(A_i UU^\top)A_i U.$$ Consider the system $$\dot V = AV,\quad  V(0) = U_0$$ where, without loss of generality, we assume that $U_0$ is of unit norm. Set $E(t) = V(t)-U(t)$. Then, differentiating $E$, we obtain  \begin{align*}\dot E &= AE+\sum_{i=1}^q\tr(A_i UU^\top)A_i U \\&= (A-\sum_{i=1}^q\tr(A_i UU^\top)A_i)E + \sum_{i=1}^q\tr(A_i UU^\top)A_iV.\end{align*} Replacing $UU^\top$ by $VV^\top -V^\top E-E^\top V +EE^\top$ in the previous equation,  we get
\begin{multline}\dot E = \left[A-\sum_{i=1}^q\tr\left(A_i (VV^\top -V^\top E-E^\top V +EE^\top)\right)A_i\right]E \\+ \sum_{i=1}^q\tr\left(A_i (VV^\top -V^\top E-E^\top V +EE^\top)\right)A_iV 
\end{multline} 
Now set $z := \|E\|$. Because $y_i >0$,  then $A_i \leq c A$ for some constant $c$ depending on the $y_i$'s; without loss of generality, we take $c \geq 1$. We let  $\lambda $ be a largest eigenvalue of $cA$, which, generically for $A_i, y_i$ is unique. Then $\|V\| \leq e^{\lambda t} \delta$ and we obtain the bound (recall that $\frac{d}{dt} \|e\| \leq \| \dot e\|$)
\begin{equation}\label{eq:oderho}\dot z \leq (\lambda+q(e^{2\lambda t}\delta^2+2e^{\lambda t}\delta z+z^2)\lambda^2)z +q\lambda^2(e^{2\lambda t}\delta^2+2e^{\lambda t}\delta z + z^2)e^{\lambda t} \delta.\end{equation}

Now let $z_1 (t)$ be the solution of the differential equation
\begin{align}\label{eq:oderho1}\dot z_1 &= (\lambda+q\lambda^2(e^{2\lambda t}\delta^2+2e^{\lambda t}\delta +1))z_1 +q\lambda^2(e^{2\lambda t}\delta^2+2e^{\lambda t}\delta z_1 + z_1)e^{\lambda t} \delta\\
&=  (\lambda+q\lambda^2(3e^{2\lambda t}\delta^2+2e^{\lambda t}\delta+e^{\lambda t}\delta +1))z_1 + q\lambda^2e^{3\lambda t}\delta^3.
\end{align} with $z_1(0)=0$.
The above is a linear time-varying ODE with positive coefficient for $z_1$ and a positive independent term $q\lambda^2e^{3\lambda t}\delta^3$. Hence, its solution is positive and, when it exists, smooth. Let $t_1>0$ be the first time for which $z_1(t_1)=1$. For all $0 \leq t \leq t_1$, $\dot z \leq \dot z_1$ and  hence $z(t) \leq z_1(t)$. The solution of equation~\eqref{eq:oderho1} is given explicitly by
\begin{multline}\label{eq:solrho1}
z_1(t)=\delta^3\lambda^2q\exp\left(t \lambda\left(\lambda q+1\right)+3\delta\lambda qe^{\lambda t}+\frac{3}{2}q\delta^2 \lambda e^{2 \lambda t})\right)\\
\int_{0}^{t} \exp\left(-\frac{\lambda}{2}\left(3 qe^{2\lambda s} \delta^2  +6 qe^{\lambda s}  \delta- 4 s + 2 \lambda qs\right)\right) ds.
\end{multline}
From the above equation, we see that for any $t^*>0$ and $\epsilon>0$ we can choose $\delta_1$ small enough so that $z_1(t^*) < \epsilon$ and thus, if moreover $\epsilon<1$,  $z(t^*) < \epsilon$. This proves the first statement.

To obtain the second the statement, let $t^*>0$ and $\epsilon>0$ be fixed and $\delta_1$ chosen so that $z_1(t^*) \leq \epsilon$. Let \begin{multline}k(\delta):= \exp\left(t^* \lambda\left(\lambda p+1\right)+3\delta\lambda pe^{\lambda t^*}+\frac{3}{2}p\delta^2 \lambda e^{2 \lambda t^*})\right)\\
\int_{0}^{t^*} \exp\left(-\frac{\lambda}{2}\left(3 qe^{2\lambda s} \delta^2  +6 qe^{\lambda s}  \delta- 4 s + 2 \lambda qs\right)\right) ds	\end{multline}
Then it is easy to see that $\min_{0 \leq \delta \leq \delta_1} k(\delta) :=k^*> 0$.  Then $z_1(t^*) \leq k^*\lambda^2q\delta^3 = O(\delta^3)$.
\end{proof}

\paragraph{Rank $1$ bottleneck for the flow}

We now put the results of the previous two paragraphs together to show that when initialized near zero, the solutions of~\eqref{eq:mainsys} will be essentially of rank $1$ for some $t_1 \geq 0$.
\begin{theorem}[Rank one bottleneck]\label{th:rk1bottle}
	Let $U(t,\delta)$ be the solution of~\eqref{eq:mainUsys} with initial state $U(0)=U_0\delta$. Then, generically for $U_0$, $A_i$ and $y_i>0$, $i=1,\ldots,q$,  for all $0<\epsilon<1$, there exists $t^*>0, \delta >0$ and $ U_1 \in \R^{n \times n}$ of rank 1  so that $\frac{\|U(t^*,\delta)-U_1\|}{\|U_1\|}  \leq \epsilon$.
\end{theorem}
\begin{proof}
Let $A = \sum_{i=1}^qy_i A_i$, and $V(t)$ be the solution of $\dot V = AV, V(0)=U_0\delta$. Then, using Lemma~\ref{lem:solinapprox},  we can find  $V_1(t,\delta)$, a rank one matrix-valued function and $t^*>0$ so that $$\frac{\|V_1(t,\delta) - V(t)\|}{\|V_1(t,\delta)\|} \leq \epsilon/2$$ for all $\delta>0$ and $t\geq t^*$.  

Using Lemma~\ref{lem:errorU}, we can find for $t^*>0$, a $\delta_1 >0$ so that ${\|U(t^*,\delta)-V(t^*,\delta)\|} \leq \epsilon/2$, for $0<\delta<\delta_1$.  Furthermore, since $\|V_1(t^*,\delta)\|= \Theta(\delta)$ by Lemma~\ref{lem:solinapprox}, and  $\|U(t^*,\delta)-V(t^*,\delta)\| = O(\delta^3)$ for $0<\delta<\delta_1$,
 by Lemma~\ref{lem:errorU}, we can find $\delta>0$ so that $ \frac{\|U(t^*,\delta)-V(t^*,\delta)\|}{\|V_1(t^*,\delta)\|} \leq \epsilon/2$. The result is now a consequence of the triangle inequality  (with $U_1:=V_1(t^*,\delta)$).
\end{proof}

Setting $X(t)=UU^\top(t)$, we obtain as Corollary:

\begin{corollary}
	Let $X(t)$ be the solution of~\eqref{eq:mainsys} with initial state $X(0)=X_0\delta$. Then for all $0<\epsilon<1$, there exists $t^*>0, \delta >0$ and $ X_1 \in \R^{n \times n}$, a symmetric positive semidefinite matrix of rank 1, so that $$\frac{\|X(t^*)-X_1\|}{\|X_1\|}  \leq \epsilon.$$
\end{corollary}

\begin{remark}\label{rem:defv1}
	From the proof of Lemma~\ref{lem:solinapprox}, and in particular from Eq.~\eqref{eq:defV1lem}, one can see that the range space of $V_1(t^*,\delta)$ is spanned the eigenvector corresponding to the largest eigenvalue of $A=\sum_{i=1}^qy_iA_i$.
\end{remark}

The following Corollary specializes the result to the case of $U_0$ of rank $1$. It will be needed below.
\begin{corollary}\label{cor:rank1botrank1}
	Under the assumptions of Th.~\ref{th:rk1bottle} and  generically for $A_i$, $y_i >0$, $1 \leq i \leq q$, and $U_0 \in \R^n$, for all $0<\epsilon<1$, there exists $t^*>0, \delta >0$ and $U_1 \in \R^{n}$ so that $\frac{\|U(t^*,\delta)-U_1\|}{\|U_1\|}  \leq \epsilon$ and $U_1$ is an eigenvector of $\sum_{i=1}^q y_i A_i$ associated with the largest eigenvalue.

\end{corollary}

\section{Compatibility of primal and regularization problems}\label{sec:compatibl}

In this second part of the analysis, we first describe conditions under which the implicit regularization for matrix factorization is exactly true. The first condition is called the rank-spread condition. Under this condition, we can exhibit a normal form for the system~\eqref{eq:mainUsys} which renders its subsequent analysis particularly transparent. As already mentioned, this condition is more restrictive than needed; we will comment on this aspect in Sec.~\ref{sec:conclusion}. We then introduce the tame spectrum assumption, which we believe is more fundamental to {\it exact} implicit regularization. Under these two assumptions, we show that the pre-critical set $\crit^* \oR$ is invariant for the dynamics~\eqref{eq:mainUsys}, and furthermore, $\crit^* \oR$ contains {\it all the minima} of $J(U)$.

\subsection{Rank spreak condition and a normal form}\label{ssec:normalform}

We denote by $r_i$ the rank of the psd matrix $A_i \in \R^{n \times n}$, $1 \leq i \leq q$. We assume  that $\sum_{i=1}^q r_i \leq n.$ To each matrix $A_i$, we assign an index set $N_i \subseteq \{1,\ldots,n\}$, $| N_i| = r_i$ such that for $i \neq j$, $N_i \cap N_j = \emptyset$. We set $N:= \cup_{i=1}^q N_i$. We can in fact choose, without loss of generality, the following assignment:   define the cumulative sums
\begin{equation}\label{eq:defmi}
m_i := \sum_{j=1}^i r_i, \quad m_0:=0\quad \mbox{and } m := \sum_{j=1}^q r_i,
\end{equation} and let \begin{equation}\label{eq:defNi}N_i:= \{m_{i-1}+1,\ldots, m_i\} \mbox{ and } r_i = |N_i|.\end{equation}  

Because the matrices $A_i$ are positive semi-definite, there exists $B_i \in \R^{n \times r_i}$ so that \begin{equation}
\label{eq:defBi}
A_i = B_i B_i^\top.\end{equation} Note that the $B_i$'s are not unique, but each is  determined up to an $O(r_i)$ symmetry. We have the following definition:

\begin{definition}[Rank spread condition]\label{def:rankspread} We say that the matrices $A_i$, $i=1,\ldots, q$, satisfy the rank spread condition if $m := \sum_{i=1}^q r_i \leq n$ and $$\dim \Span\{ B_i, i=1,\ldots,q\} =m,$$ where the $B_i$ are as in Eq.~\eqref{eq:defBi}.
\end{definition}
This condition is met generically in the under-determined regime. It is easy to see that the  definition above is independent of the particular choice of $B_i$'s. The following result is key in establishing a normal form for the gradient flow~\eqref{eq:mainUsys}. Recall that $e_i$ is the vector in $\R^n$ with all entries equal to 0 except for the $i$th one, which is equal to 1.
\begin{lemma}\label{lem:P}
Let $A_i$ be positive semi-definite matrices of rank $r_i$ satisfying the rank spread condition, $i=1\ldots q$, and let $N_i \subseteq \{1,\ldots,n\}$ be given as in Eq.~\eqref{eq:defNi}. Denote by $$L_i:= \Span \{e_j \mid j \in N_i\},$$ and let $L:= \oplus_{i=1}^q L_i$. Then  there exist an invertible matrix $P \in \R^{n \times n}$ with the following properties, for $1 \leq i \leq q$:
\begin{enumerate} 
\item $P^\top A_i P$ is the identity on $L_i$, and has  $L_i^\perp$ as kernel.
\item The matrices $M_i:= P^{-1} A_i P$ define injective maps $M_i:L_i \to L$ and, in particular, $M_i L_i^\perp = 0$.
\item The matrix $P^\top P$ is block diagonal, with leading block of size $m$, and lower block equal to the $(n-m) \times (n-m)$ identity matrix. Furthermore, the leading block of $P^\top P$ is the inverse of the leading block of $\sum_{i=1}^qP^{-1}A_iP$.

\end{enumerate}
\end{lemma}
\begin{proof}
Because the $A_i$ are positive semidefinite of rank $r_i$, we can write $A_i = B_iB_i^\top$, for some $B_i \in \R^{n \times r_i}$. The rank spread condition guarantees that $m=\sum_{i=1}^q r_i \leq n$. Hence, we can define the matrix  $B  \in \R^{n \times m}$ with columns equal to the columns of  the matrices $B_i$, with columns $m_0+1$ to $m_1$ taken from $B_1$, $m_1+1$ to $m_2$ taken from $B_2$, etc., where the $m_i$'s were defined in Eq.~\eqref{eq:defmi}. Set $B^\perp \in \R^{n \times (n-m)}$ to be a matrix with orthonormal columns  spanning the orthogonal subspace of the column span of $B$: namely, $B^\perp$ satisfies $$ (B^\perp)^\top B^\perp = I\mbox{ and } B^\top B^\perp  =0.$$ We now define

\begin{equation}\label{eq:defP}P := \left[ B (B^\top B)^{-1}\quad B^\perp \right] \in \R^{n \times n}.\end{equation} 
By construction, $P$ is invertible and  $P^\top$ maps each column vector of the matrices $B_i$ to (necessarily distinct) vectors of the canonical basis of $\R^n$. In particular, $$\Span\{P^\top B_i\} = \Span \{e_j \mid j \in N_i\}.$$ Writing $A_i$ as $B_i B_i^\top$ and using this fact, we obtain the first item.

The second and third items follow directly from an evaluation of the matrix products. For the second item, it is helpful to first  verify that  $P^{-1}$ is equal to  $$P^{-1} = \begin{bmatrix} B^\top  \\ (B^\perp)^\top  \end{bmatrix}.$$ \end{proof}

In coordinates, the first item states that $P^\top A_i P$ is a diagonal matrix with all  entries equal to 0, save for the diagonal entries indexed in $N_i$, which are equal to 1. Recalling the definition given in Eq.~\eqref{eq:def:Ei}, we have $$P^\top A_i P = E_i.$$ 
The second item states that $P^{-1}A_iP$ is a matrix whose last $n-m$ rows are equal to 0, and whose columns are all 0 save for the columns indexed in $N_i$.

 \paragraph{Normal form} Relying on Lemma~\ref{lem:P}, we now construct a normal form for the flow of Eq.~\eqref{eq:mainUsys}. We do so in the general case of $X$ of rank $k$, corresponding to $U \in \R^{n \times k}$, since we often will use the case $k=1$ below. 
Recall that $\rho_i(U)=\tr(A_iUU^\top)-y_i$ and the flow of interest is $$ \dot U =- \sum_{i=1}^q \rho_i A_iU.$$ 
We assume that the $A_i, i=1,\ldots, q$ satisfy the rank spread condition of Def.~\ref{def:rankspread}. Let $P \in \R^{n \times n}$ be as in Lemma~\ref{lem:P} and introduce $\bar U \in \R^{n \times k}$ satisfying $$P\bar U = U,$$ then the above equation becomes
\begin{equation}\label{eq:dynbU}
\dot {\bar U} = -	\sum_{i=1}^q \rho_i(P \bar U) P^{-1}A_iP \bar U.
\end{equation}
From item (1) in Lemma~\ref{lem:P}, we have that $$\rho_i(P \bar U)=\tr(\bar U^\top P^\top A_i P \bar U)-y_i = \tr(\bar U^\top E_i \bar U)-y_i = \sum_{j \in N_i} \sum_{l=1}^k \bar u_{jl}^2,$$ 
hence $\rho_i(P \bar U)$ depends only on the {\it rows} of $\bar U$ indexed by $N_i$. From item (2) in  Lemma~\ref{lem:P}, we have that $P^{-1}A_iP$  has kernel $L_i^\perp$. 

Putting the above two observations together, we have that $\rho_i(P\bar U) P^{-1}A_iP\bar U$ only depends on the entries $\bar u_{jl}$ of $\bar U$ with  $j \in N_i$. Since $E_j$ maps into $L_j$ and $E_j^2=E_j$, we have the relation   $$P^{-1}A_iP E_j = P^{-1}A_i P \delta_{ij},$$ where $\delta_{ij}$ is the Kronecker delta (i.e., $\delta_{ij}=1$ if and only if $i=p$, and is zero otherwise).  We thus have the following relation:  $$\sum_{i=1}^q \rho_i P^{-1}A_i P = \sum_{i=1}^q  P^{-1}A_i P \sum_{j=1}^q \rho_j E_j.$$  Observe that  $\sum_{j=1}^p \rho_j E_j$ maps $L \to L$ and can be expressed over this space as a diagonal matrix $D(\bar U)$ with $m$ non-zero entries, and with $\rho_i$ on the diagonal entries indexed by $N_i$.  From item 3 in Lemma~\ref{lem:P}, we see that $\sum_{i=1}^p  P^{-1}A_i P$ maps $L$ to itself,  and $L^\perp$ to itself as well. Furthermore, when restricted to $L$, $\sum_{i=1}^p  P^{-1}A_i P$ can be expressed  as a matrix $Q \in S_m$ (in fact, a direct calculation using the explicit form for $P$ given in the proof of Lemma~\ref{lem:P} shows that $Q = B^\top B$.), and when restricted to $L^\perp$, it is the zero map.

With the above observations in mind, set \begin{equation}\label{eq:defVnormal}x :=\begin{pmatrix} \bar u_{11} & \cdots& \bar u_{1k} \\
\vdots & & \vdots \\
\bar u_{m1} & \cdots & \bar u_{mk}	
\end{pmatrix}.\end{equation} The {normal form} comprises two sets of equations. The first is 
\begin{equation}\label{eq:nomsys} \frac{d}{dt} x = -QDx,
\end{equation} where $Q=B^\top B$ as described above, and  $$D:=D(x)=\sum_{i=1}^q (\tr(x^\top E_i x)-y_i)E_i$$ where, with a slight abuse of notation, we set $E_i = \sum_{j \in N_i} e_je_j^\top$ but with $e_j \in \R^m$ and $\rho_i(x)=\tr(x^\top E_i x)-y_i$. The second set of equations deals with the variables in the rows of $\bar U$ below the  $m$th row (if there are any) and is given by \begin{equation}\label{eq:dynbarUzero}\frac{d}{dt} {\bar u}_{jl}=0\quad  \mbox{ for } j \notin \cup_{i} N_i, l=1,\ldots,k .\end{equation}

In summary, the {\bf normal form or normal system} is given by Eqns.~\eqref{eq:nomsys} and~\eqref{eq:dynbarUzero}; it is obtained by changing variables, and observing that in the new variables, the dynamics of a subset of the variables is given by Eq.~\eqref{eq:nomsys}, and the dynamics of the remaining variables is zero. 
\paragraph{The regularization problem in normal coordinates}

We now write the optimization problem $\oR$ in the normal coordinates. We again working in the case of arbitrary rank $k$ (see~\eqref{eq:optimprobgenU}) \begin{equation}\label{eq:optbarU}:
\oR_k: \min_{ U \in \R^{n\times k}} \|U\|^2 \quad \mbox{ s.t. } \tr (A_i UU^\top) = y_i, \quad i=1,\ldots,q.
\end{equation}  Let $P$ be as in the statement of the Lemma~\ref{lem:P} and let $E_i = \sum_{j \in N_i} e_je_j^\top$ for the $N_i$ defined in~\eqref{eq:defNi}.
\begin{proposition}\label{prop:optnormcoord}
Consider the minimization problem over $\R^{m \times k}$ \begin{equation}
\label{eq:optx}
\oR_k: \min_{x \in \R^{m\times k}} \tr(x^\top Q^{-1}x) \quad \mbox{s.t. } \tr(x^\top E_i x) =y_i, i=1,\ldots,q\end{equation} where $Q^{-1}$ is the leading $p \times p$ block of $P^\top P$, $P$ as is Lemma~\ref{lem:P},  and let $x^*$ be a minimizer. Then $ U^*:=\begin{pmatrix} (Px^*) ^\top &  0 \end{pmatrix}^\top \in \R^{n\times k}$  is a minimizer of the problem of Eq.~\eqref{eq:optbarU}.
\end{proposition}
 We emphasize that the matrix $Q$ appearing in the above lemma is the one of the normal system of Eq.~\eqref{eq:nomsys}. With a slight abuse of notation, we refer to the problem above as $\oR$ as well, since it is related to the one of Eq.~\eqref{eq:optimprobgenU} by a change of variables.

\begin{proof}
Starting with the problem of Eq.~\eqref{eq:optimprobgenU}, and setting $P\bar U = U$, we get that it is equivalent to 
$$\min_{\bar U \in \R^{n\times k}} \tr(\bar U^\top P^\top P \bar U) \quad \mbox{s.t. } \tr(\bar U^\top P^\top A_i P \bar U) = y_i,\quad 1\leq i \leq q.$$
Using item~1 of Lemma~\ref{lem:P}, the constraints become $\tr(\bar U^\top E_i \bar U) = y_i.$ Observe that the variables $\bar u_{jl}$ are not constrained if $j > m$, where we recall that $m$ is defined in Eq.~\eqref{eq:defmi}. Furthermore, from item~3 of Lemma~\ref{lem:P}, $P^\top P$ is block diagonal with  a leading block of size $m$, which we denoted by $Q^{-1}$,  and a lower principal block equal to the identity matrix.  Hence $$\tr (\bar U^\top P^\top P \bar U) = \tr \left( (\bar u)_{1..m,1..k} Q^{-1}(\bar u)_{1..m,1..k}^\top\right)+ \tr \left( (\bar u)_{m+1..k,1..k} (\bar u)_{m+1..k,1..k}^\top\right)$$ 
We conclude from the above equation  that a constrained minimizer is so that $$\bar u_{jl} =0,\quad \mbox{ for } j=m+1,\ldots,n, l=1,\ldots,k$$ Letting $x \in \R^{m \times k}$ be  $$x=(\bar u_{il})_{i=1,\ldots,m, l= 1,\ldots, k},$$ we recover the problem stated in Eq.~\eqref{eq:optx}. 
\end{proof}

\paragraph{A Morse-Bott function and a metric} We now derive the {\it function} and the {\it inner product} for which the flow in normal variables is a gradient. We furthermore show that this function is a so-called {\it Morse-Bott function}, i.e., a $C^2$ function whose critical set is a closed manifold, and whose  Hessian evaluated at any point of the critical set has a  kernel  equal to the tangent space to the critical set at this point. Unless explicitly mentioned,  we ignore the variables $\bar u_{jl}$ with $j >m$ when we refer to the normal form. This can be done without loss of generality since the dynamics of these variables is trivial.

Recall that the inner product induced by $Q$ is defined as $\langle x_1,x_2\rangle := x_1^\top Q x_2,$ for $x_1,x_2 \in \R^n$ and that the gradient of $J$ for this inner product is $$\grad J = Q^{-1} \frac{\partial J}{\partial x}.$$

\begin{theorem}\label{th:morsebott}
Let the sets $N_i$, $i=1,\ldots,q$ be disjoint and so that $\cup_{i=1}^q N_i=\{1,\ldots,m\}$ and $k \in \{1,\ldots, n\}$. Consider the normal dynamics $$\dot x = -QD(x)x,$$ where $Q \in \R^{m \times m}$ is a positive definite matrix, $x \in \R^{m \times k}$  $D(x)=\sum_{i=1}^q\rho_i E_i$ with  $y_i >0$, $\rho_i(x) = \tr( x^\top E_ix)-y_i$, and $E_i = \sum_{j \in N_i} e_j e_j^\top$. Define the function $$J_k(x) = \frac{1}{4}\sum_{i=1}^q (\tr(x^\top E_i x)-y_i)^2.$$ Then
\begin{enumerate}
\item the normal dynamics is the gradient flow of $J_k$  for the inner product induced by $Q^{-1}$. 
\item The function $J_k$ is a Morse-Bott function.
\item\label{it:propminrho} The set of local minima of $J_k$ is given by $x \in \R^{m\times k}$ such that $\rho_i(x)=0$, $i=1,\ldots, q$.
\end{enumerate}
\end{theorem}
We denoted the set of local minima of $J_k$ by $\crit_0 J_k$. The above Proposition thus says that 
\begin{equation}\label{eq:defCS1}
\crit_0 J_k:= \{x \in \R^{m \times k} \mid \rho_i(x) = 0, i=1\ldots, q \}.
\end{equation}

\begin{proof}
With this notation, it is easy to verify that for $l=1,\ldots,k$, \begin{equation}\label{eq:zergrad}\frac{\partial J_k}{\partial x_{jl}} = \rho_i  x_{jl}\quad \mbox{ if } j \in N_i,\end{equation} from which we obtain that $\frac{\partial J_k}{\partial x} = \sum_{i=1}^q \rho_i E_ix = D(x)x$. This proves the first statement.

Since $Q$ is non-degenerate, we obtain that $\grad J_k=0$ if and only if
\begin{equation}\label{eq:defzeroJ} \rho_i(x) x_{jl} =0\quad \mbox{ for all } l=1,\ldots,k, j\in N_i, i=1\ldots, q.\end{equation} In order to verify that $J_k$ is a Morse-Bott function,
we need to verify that its zero set is a closed submanifold of $\R^{n\times n}$, that the Hessian of $J_k$ is non-degenerate at isolated critical points, and that the kernel of the Hessian spans the tangent space at the critical submanifolds.

 From Eq.~\eqref{eq:defzeroJ}, we have that the critical set of $J_k$ is given by the intersection of $q$ subsets given by either $\rho_i(x)=0$, or  $x_{jl}=0, j \in N_i, l=1,\ldots,k$, for $i=1,\ldots,q$.  We see that the zero-set of $\rho_i(x)=\sum_{j \in N_i}\sum_{l=1}^k x_{jl}^2 -y_i$ is a {\it closed} subset of $\R^{m \times nk}$---in fact a sphere of radius $\sqrt{y_i}$ contained in $\R^{r_i \times k}$, where we recall that $|N_i|= r_i$---and so is the linear subspace defined by $x_{jl}=0, j \in N_i, l=1,\ldots,k$. We conclude that $\crit J_k$ is the intersection of closed sets and thus is closed.

To evaluate the Jacobian of $\frac{\partial J_k}{\partial x_{jl}}$, it is easier to first write this matrix as a vector. We do so in a {\it row first} fashion. Hence we now represent $x$ as a column vector  $X \in \R^{mk}$ with entries $$X=(x_{11}, x_{12}, x_{1k},x_{21},\ldots,x_{mk})^\top.$$ A short calculation shows that with this notation, we have \begin{equation}\label{eq:gradJX}\frac{\partial J_k}{\partial X} = (D \otimes I_n) X=:D_1X,\end{equation} where $\otimes$ is the Kronecker product.
Now recall the definition of the sets $N_i$, we denote by $M_i$ their counterparts after vectorization; more precisely, $M_i$ contains the indices of the the $X_j$ who were in a row  with index in $N_i$: $$j \in M_i \Leftrightarrow X_j = x_{i'l}\quad \mbox{ for some } i' \in N_i, l=1,\ldots,k.$$ The set $M_i$ has cardinality $k r_i$. As before, we let $E_i \in \R^{mk \times mk}$ be the diagonal matrix with zero entries except for the ones indexed by $M_i$, which are one, and we set $\bar X_i = E_i X$. 

The matrix $D_1$ is a diagonal matrix, with entries $(D_1)_{jj} = \rho_i $ if $j \in M_i$, or said otherwise, $$D_1 = \sum_{i=1}^q \rho_i E_i.$$ From Eq.~\eqref{eq:gradJX}, the critical set of $J_k$ is easily seen to be defined by the intersection of the zero sets $\rho_i(X) X_j =0$, $j \in M_i$.

Differentiating $\rho_i(X)$, we obtain  $\frac{\partial \rho_i}{\partial X} = \sum_{j \in M_i}2X_j e_j$. Hence, a short calculation shows that

\begin{multline}\frac{\partial^2 J}{\partial X^2} = \sum_{i=1}^q\left[ \rho_i E_i + \sum_{j \in M_i} 2X_j e_j (E_iX)^\top \right]=\sum_{i=1}^q\left[ \rho_i E_i + \sum_{j,l \in M_i} 2X_jX_l e_je_l\right] \\
=\sum_{i=1}^q \left[\rho_i E_i + 2 \bar X_i \bar X_i^\top \right] \end{multline} where we used the facts that $E_i = \sum_{l \in M_i} e_le_l^\top$ and $e_l^\top X = X_l$.  The previous relation shows that the matrix of second derivatives is block diagonal, with blocks of size $|M_i|$, and that, additionally, block $i$ only depends on  $X_j$ with $j \in M_i$. Therefore, it is sufficient to verify that the non-degeneracy condition holds for each block of $\frac{\partial ^2J}{\partial X^2}$\footnote{Alternatively, one could immediately argue that is is sufficient to verify that each term of $J$ is Morse-Bott, since the terms do not share variables.}.
Let $X$ be a zero of $\grad J$, and first assume that $\rho_i(X)=0$. The corresponding zero-set is then a  sphere of radius $\sqrt{y_i}$ in  $\R^{|M_i|}$) and the $i$th block of $\frac{\partial J_k^2}{\partial X^2}$ is $\bar X_i \bar X_i^\top$. When evaluated on  the set $\rho_i(X)=\rho_i(\bar X_i)=0$, the vector $\bar X_i$ is clearly non-zero, and is {\it normal} to the set $\rho_i(X)=0$. Hence the kernel of $\bar X_i \bar X_i^\top$ is exactly the tangent space of $\rho_i(\bar X_i)=0$.

Now assume that $X_j =0$ for all $j \in M_i$. Then the zero-set is $0$ dimensional in $\R^{|M_i|}$. The  $i$th block of $\frac{\partial^2 J_k}{\partial X^2}$ is $-y_iE_i$, which has no kernel in $\R^{|M_i|}$ as required. This shows that $J_k$ is a Morse-Bott function.

To prove the last part, it suffices to observe that at any critical point $X \in \R^{mk}$ such that $\rho_i(X)=0$, the corresponding block of the Hessian is $\bar X_i\bar X_i^\top$, which is positive semi-definite. Hence the Hessian at the points belonging to the intersection of the sets defined by $\rho_i(X)=0$, $i=1,\ldots,q$ is positive semi-definite, and thus these points are minima. Reciprocally, for critical points such that  $\bar X_i = 0$ for any $i$,  the corresponding block of the Hessian is  $ -y_i E_i$, which is negative definite. Such points cannot be minimizers.
\end{proof}

The critical set of $J_k$ can be visualized geometrically with ease. Consider the vectorized coordinates described in the proof of Theorem~\ref{th:morsebott}. We can write $\R^{mk}$ as the product $\R^{r_1k} \times \R^{r_2k} \times \cdots \times \R^{r_qk}.$ For each $i=1,\ldots,q$, choose either $\rho_i(\bar X_i)=0$, which is a sphere of radius $\sqrt{y_i}$ in $\R^{r_ik}$, or $\bar X_i=0$, which is the origin of $\R^{rn_i}$. The cross product of these $q$ elements is a component of the critical set. Alternatively, we can consider $\rho_i(X)=0$ to be a subset of $\R^{mk}$ (i.e., a sphere cross-product the plane spanned by $e_j, j \notin M_i$), and similarly for $\bar X_i=0$ (i.e., the plane spanned by $e_j, j \notin  M_i$). A component of the critical set is then the intersection of these sets in $\R^{mk}$. The two points of view are equivalent. For example, take $m=3, k=1$ and $N_1=\{1,2\}$, $N_2=\{3\}$, and put coordinates $(x,y,z)\in \R^3$ on the state-space. The component of the critical set given by $\rho_1=0$, $\rho_2=0$ is then the union of two disjoint circles of radius $\sqrt{y_1}$, centered around the $z$ axis, and in the planes $z=\pm \sqrt{y_2}$. The component of the critical set given by $\rho_1=0$ and $\bar X_3 = 0$ is a circle of radius $\sqrt{y_1}$ centered at the origin and in the plane $z=0$. Similarly, if $m=4,k=1$ and $N_1=\{1,2\}$, $N_2=\{3,4\}$, the zero set determined by $\rho_i=0$, $i=1,2$ is a torus (the product of two circles) in $\R^4$.

The previous  Theorem shows the following important fact:  the set of local minima of $J_k$ is exactly the feasible set of $\oR$. Recalling that a gradient flow converges generically for the initial condition to a local minimum, we have as a Corollary that the flow of the primal problem will converge generically to a feasible point of $\oR$:
\begin{corollary}	
\label{cor:convnormaldyn}
Consider the normal system of Eq.~\eqref{eq:nomsys} $$\dot x =- QDx,$$ with $x \in \R^{m \times k}$. Then $\feas(\oR_k) = \crit_0(J_k)$ and, in particular,  generically for $x(0)$, the solution converges to $x^* \in \feas(\oR_k)$.
\end{corollary}
\subsection{The case of commuting generalized projection matrices.}
As a direct consequence of the construction of the normal form, we can show that various forms of implicit regularization take place in the particular case of commuting {\it generalized} projection matrices (defined below).  This extends on the result of~\cite{implicitregmatr2017} insofar as we do not require the initial condition of the flow to be a multiple of the identity.
Recall that a symmetric matrix $A$ is a projection matrix if $A^2=A$. This implies, in particular, that the spectrum of $A$ only contains $0$ and $1$. We say that $A$ is a {\bf generalized projection matrix} if $A^2 = \gamma A$ for some positive number $\gamma$. In particular, note that all rank $1$ psd matrices are generalized projection matrices. 
 \begin{proposition}[Commuting generalized projection matrices]\label{prop:commproj}
Assume that  $A_i$, $i=1,\ldots,q$ are generalized projection matrices satisfying the rank spread condition, and that they pairwise commute.  Then, generically for $X_0$, for all $\varepsilon >0$, there exists $\delta >0$ so that $\varphi_\infty(X_0\delta)$ is $\epsilon$-close to a minimizer of~\eqref{eq:optimprobgen}.
\end{proposition}
The above proposition says that the regularization problem $\oR$ of Eq.~\eqref{eq:defoS} approximately regularizes the  main system~\eqref{eq:mainsys}. It holds true for the rank $k$ of $X$ between $1$ and $n$.

\begin{proof}We work in normal coordinates. Starting from the primal problem in $U$ variables~\eqref{eq:mainUsys}, we introduce $P \bar  U = U$, with $P$ as in Lemma~\ref{lem:P} and the normal variables $x \in \R^{m \times k}$ (see Eq.~\eqref{eq:defVnormal}), and $\bar u_{jl}$, $j >m$. 

Now write $A_i=B_iB_i^\top$, $i=1,\ldots, q$ for some $B_i\in \R^{n \times r_i}$. Without loss of generality, we can assume that the $B_i$ have orthogonal columns and, as a consequence of the generalized projection assumption, these columns have necessarily the same norm.  The fact that the $A_i$'s pairwise commute tells us that  $B_iB_i^\top B_jB_j^\top =B_jB_j^\top B_iB_i^\top$, and the fact that they satisfy the rank spread condition tells us  that $\Span\{B_i\} \cap \Span\{B_j\} = \{0\}$. From the above two facts, we obtain that  $B_i^\top B_j =0$ and thus conclude that  $B_i$ and $B_j$ have orthogonal columns. Therefore, the corresponding matrix $Q$ is diagonal, with the diagonal entries $q_{jj}=q_{j'j'}$ equal for $j,j' \in N_i$, and   $Q^{-1}$ has the same form.

From  Proposition~\ref{prop:optnormcoord}, we know that the minimizers in normal coordinates are so that $\bar u_{jl}=0$ for $j > m$ and $l=1,\ldots, k$, and $x$ is a minimizer of $\tr(x^\top Q^{-1}x)$.  From the form $Q^{-1}$ described above, we have \begin{equation}\label{eq:propinter1}\tr(x^\top Q^{-1}x) = \sum_{i=1}^q q_{ii}^{-1} \sum_{j\in N_i,l=1,\ldots,k} x_{jl}^2.\end{equation} Recall that the constraints of the regularization problem $\oR$ are $\sum_{j \in N_i} \sum_{l=1}^k x_{jl}^2 = y_i$. Hence, for {\it any} matrix $x$ satisfying the constraints, the cost is $\tr(x^\top Q^{-1}x)=\sum_{i=1}^q q_{ii}^{-1} y_i$. Thus the minima of $\oR$ are so that $x$ is feasible for $\oR$, and $\tilde u_{jl}=0$, $j>m, l=1,\ldots,k$. 

We now show that the primal problem converges to a point arbitrarily close to global minimum of $\oR$. We have that $x$ obeys the equation of the normal system~\eqref{eq:nomsys} \begin{equation}\label{eq:dynVrecall}\dot x = -QD(x)x\end{equation} with $ D(x)=\sum_{i=1}^m (\tr(x^\top E_i x)-y_i)E_i=\sum_{i=1}^m \rho_i E_i$ and from Eq.~\eqref{eq:dynbarUzero} $$\tilde u_{jl}(t)=\bar u_{jl}(0),\quad \mbox{ for } m <j \leq n.$$ 
From Corollary~\ref{cor:convnormaldyn}, we know that the system~\eqref{eq:dynVrecall}  converges generically to $x^* \in \R^{m \times k}$ so that $\tr\left((x^*)^\top E_i x^*\right) =\sum_{j \in N_i} \sum_{l=1}^n x_{jl}^2 = y_i, \quad i=1,\ldots, q,$ i.e., $x^*$ is feasible for $\oR$. Choosing $\delta$ small enough, $u_{jl}$, $j>m$ is as small as needed. 
\end{proof}

The following corollary can be extracted from the proof of Prop.~\ref{prop:commproj}: for  $m=n$, the convergence is global, and not only for small initial conditions:
\begin{corollary}[Commuting generalized projection matrices]\label{cor:commproj}
Assume that  $A_i$, $i=1,\ldots,q$ are generalized projection matrices of respective ranks $r_i$, satisfying the rank spread condition, that they pairwise commute and that $m = \sum_{i=1}^q r_i=n$   Then, generically for $X_0$,  $\varphi_\infty(X_0\delta)$ is $\epsilon$-close to a minimizer of~\eqref{eq:optimprobgen}.
\end{corollary}

\subsection{Tame spectrum assumption and compatibility of primal and regularization problems }

We now present what we believe is the main mechanism at the heart of implicit regularization for matrix factorization. In the previous case, namely  the case of commuting projection matrices $A_i$,  once the flow converged to the constraint set $\feas(\oR)$---and we showed this happened generically for the initial condition since the flow was gradient and its set of minima was equal to the feasible set of the regularization problem---the fact that $\delta$ was small guaranteed that the system converged to {\it near} a minimizer of $\oR$. Said otherwise, in normal coordinates, the role of the small initial condition was particularly transparent and, in particular, results of Sec.~\ref{sec:bottleneck} were not needed.

Of course, this mechanism by itself does not explain the implicit regularization phenomenon. We exhibit in this section a different, and more complex, dynamical process taking place following the discussion of Section~\ref{sec:impltheory}. 

To this end, we will introduce the so-called {\it tame spectrum} assumption. Essentially, this assumption identifies a set $M$ for which primal and regularization are compatible.
Throughout this subsection, we assume that $X$ is {\it of rank 1}, or equivalently that $k=1$ and $U \in \R^n$. We now introduce the tame spectrum assumption. 

\paragraph{The tame spectrum assumption} We start with the following simple lemma:
\begin{lemma}\label{lem:defleade}
	Let $A_i$, $i=1,\ldots, q$ be psd matrices satisfying the rank spread condition and  so that $\sum_{i=1}^q A_i$ has spectrum $\{\alpha',\beta',\ldots,\beta', 0, \ldots, 0\}$. Then the matrix $Q$ of the normal system~\eqref{eq:nomsys} has spectrum $\{\alpha',\beta',\ldots,\beta'\}$.
\end{lemma}
\begin{proof}
Let $B_i \in \R^{n \times r_i}$ be such that $A_i=B_iB_i^\top$. Let $B \in \R^{n \times m}$, $m = \sum_{i=1}^q r_i$,  be the matrix whose columns are the columns of the $B_i$. On the one hand, from the proof of Lemma~\ref{lem:P},  we know that $Q=B^\top B$ and from the rank spread condition, $Q$ is of full rank. On the other hand, $\sum_{i=1}^q A_i = BB^\top$, from which the result follows. 
\end{proof}

It is easy to see that when $Q\in \R^{m \times m}$ is symmetric positive definite and has a spectrum as $\{ \alpha', \beta', \ldots, \beta'\}$, then there exists a vector $v \in \R^m$ and constants $\alpha> \beta >0$ so that $Q$ can be expressed as \begin{equation}\label{eq:defQleading}Q:= \alpha vv^\top + \beta I.\end{equation} We call $v$ the {\bf leading eigenvector} of $Q$. Note that in Lemma~\ref{lem:defleade}, we can replace the assumption that the $A_i$ satisfy the rank spread condition with the requirement that $\sum_{i=1}^q A_i$ has rank $m$, where we recall that $m =\sum_{i=1}^q r_i$ and $r_i = \operatorname{rank} A_i$. The parameter $\alpha$ is the {\it spectral gap} of $Q$.

\begin{definition}[Tame spectrum assumption]\label{def:LE}
	We say that the positive semidefinite matrices $A_i$, $i=1,\ldots,q$ satisfy the tame spectrum assumption if $\sum_{i=1}^q A_i$ is a psd matrix of rank $m$ with spectrum $\{\alpha',\beta',\cdots,\beta',0,\cdots 0,\}$, with $\alpha' , \beta' > 0$, and the corresponding leading eigenvector of $Q$ is so that $\|\bar v_i\|^2 := v^\top E_i v \neq 0$, $i=1,\ldots,q$.
\end{definition}
The condition that $\|\bar v_i\| \neq 0$ is generic for the $A_i$. The role of this assumption is to weed out particular cases, requiring lengthy computations, in the proofs below.

Recall the definition of the vector space $\cD_N$ in Sec.~\ref{ssec:notation}. The elements of $\cD_N$ can be written as $\sum_{i=1}^q \nu_i E_i$, with $\nu_i \in \R$ and $E_i$ as in Eq.~\eqref{eq:def:Ei}. The following vector subspace of $\R^m$ will play an important role: given a non-zero vector $v \in \R^m$, we define

\begin{equation}\label{eq:defLv}\Lambda_v := \left\{ x \in \R^m \mid  x= \Lambda v \mbox{ for some } \Lambda \in \cD_N \right\}.
\end{equation}
\paragraph{Compatibility of primal and regularization problems} We now show that under the tame spectrum assumption, the primal and regularization problems are compatible. We do so in three steps:
\begin{lemma}\label{lem:invdyn1}
Assume that the $A_i$, $i=1,\ldots,q$, satisfy the tame spectrum assumption, with leading eigenvector $v \in \R^m$. The normal dynamics of Eq.~\eqref{eq:nomsys}, with $x\in \R^m$, leaves $\Lambda_v$ invariant, where $v$ is the leading eigenvector of $Q$.
\end{lemma}
\begin{proof}
Let $x \in \Lambda_v$ and $\Lambda \in \cD_N$ so that $x = \Lambda v$.  Since $\Lambda_v$ is a vector space, $T_x\Lambda_v  = \Lambda_v$ for all $x \in \Lambda_v$. The dynamics of Eq.~\eqref{eq:nomsys} is the sum of two terms, $\alpha vv^\top D(x)x$ and $\beta D(x)x$. The first term is clearly in $T_x\Lambda_v$. Since $D(x)\in \cD_N$, so is $D(x)\Lambda$ and we conclude that the second term  is  in $T_x\Lambda_v$ as well, which proves the result.
\end{proof}

The next result show that the set of minima of $J_1(x)$---we denoted that set by $\crit_0 J_1(x)$---and $\Lambda_v$ intersect transversally and that moreover the intersection is a finite set of points.  This result will be key to study the dynamics of the primal problem in $\Lambda_v$. 

\begin{lemma}\label{lem:cor:transv}Assume that the  tame spectrum assumption with              leading eigenvector $v$ holds.  Then $\Lambda_v$ intersects the set $\crit_0 J_1$ transversally, and this intersection is a finite set of points of cardinality $2^q$.
	\end{lemma}

\begin{proof}
Let $x = \Lambda v \in \Lambda_v$. Then since $\Lambda = \sum_{i=1}^q \lambda_i E_i$ for some $\lambda_i \in \R$, and since $\rho_i(x)=\rho_i(E_i x)$ by definition of $\rho_i$, we have that points in the intersection $x\in \Lambda_v \cap\crit_0 J_1(x)$ are solutions of $\rho_i(\bar x_i)=\rho(\lambda_i \bar v_i)=0$ or, equivalently,  $$\lambda_i^2 \|\bar v_i\|^2 = y_i, \quad i=1,\ldots, q.$$ Hence, there are $2^q$ points of intersection, characterized by $\lambda_i = \pm \sqrt{y_i}/\|v_i\|$, and these points are pairwise distinct. 

To see that the intersection is transversal, recall that $\crit_0 J_1$ is a product of $q$ spheres $S^{|N_i|-1}\subset \R^{|N_i|}$, each of codimension one in $\R^{|N_i|}$. Similarly, $\Lambda_v$ is the product of $q$ lines $\lambda_i \bar v_i \subset \R^{|N_i|}$. A line through the origin in Euclidean space always intersects a sphere centered at the origin transversally, which proves the claim.
\end{proof}

The importance of $\Lambda_v$ stems from the following observation. Consider the optimization problem in normal coordinates, described in Eq.~\eqref{eq:optx}. Its pre-critical set is given by \begin{equation}\label{eq:ELoptx}\crit^* \oR_1 = \{ x \in \R^m \mid  x = Q\Lambda x,\quad \mbox{ for some } \Lambda \in \cD_N\}\end{equation}
Given a vector $v \in \R^m$, we denote by $v^\perp$ its orthogonal subspace in $\R^m$. Namely, $$v^\perp = \{ x \in \R^m \mid v^\top x =0\}.$$ We can express $\crit^* \oR_1$ explicitly as follows:
\begin{lemma}\label{lem:crits1}
Assume that the tame spectrum assumption holds with leading eigenvector $v \in \R^m$. Then $$\crit^* \oR_1 = \Lambda_v \cup v^\perp.$$	
\end{lemma}
Note that $\Lambda_v$ and $v^\perp$ intersect generically at more than $\{0\}$, since $v^\perp$ is of dimension $m-1$.
\begin{proof}
Let $ x \in \crit^* \oR_1$. Then $x$ satisfies
$$(I -\beta \Lambda) x =\alpha vv^\top \Lambda x$$ for some $\Lambda \in \cD_N$. Assume first that $(x, \Lambda)$ is such that $v^\top \Lambda x =a \neq 0$, i.e. $\Lambda x \notin v^\perp$. Then $x$ satisfies $$(I-\beta \Lambda)x = a v.$$ Since $\bar v_i \neq 0$ by the tame spectrum assumption, and since $I-\beta \Lambda \in \cD_N$, we conclude that $I-\beta \Lambda$ is	 invertible, with an inverse in $\cD_N$,  and $x \in \Lambda_v$.

Now assume that $(x,\Lambda)$ is such that $\Lambda x \in v^\perp$. Then $a=0$ and $x$ satisfies $x = \beta \Lambda x$. Plugging this relation in $v^\top \Lambda x=0$, we obtain that $v^\top x =0$ and thus $x \in v^\perp$, which concludes the proof.
\end{proof}

 The following lemma says that the minimal values of $\oR_1$ are obtained for $x \in \Lambda_v$, thus  it will be sufficient to consider the component $\Lambda_v$ of $\crit^* \oR_1$.
\begin{lemma}\label{lem:Lvincl}Assume that the tame spectrum assumption holds with leading eigenvector $v\in\R^m$. Consider the cost function $K$ of optimization problem $\oR_1$ and set $a_1:=\min K \mbox{ s.t. } x\in v^\perp \cap \feas(\oR_1)$ and $a_2= \min K \mbox{ s.t. } \Lambda_v \cap \feas(\oR_1)$. Then $a_1 \geq a_2$. 
\end{lemma}

If either intersection in the Lemma statement is empty, the corresponding $a_i$ is set to $+\infty$.
\begin{proof}
Under the assumptions of the Lemma, we have that $Q=\alpha vv^\top +\beta I$. Using this expression for $Q$ in Eq.~\eqref{eq:ELoptx}, we obtain that $\oR_1$ is 
\begin{equation} \min_{x \in \R^m} K(x):=x^\top (-\alpha' vv^\top + \beta I) x  \quad \mbox{ s.t. } x^\top E_i x = y_i, i = 1,\ldots, q\end{equation} where $\alpha', \beta >0$. If $x \in v^\perp \cap \feas(\oR_1)$, then the cost reduces to $K(x)=\beta \|x\|^2 = \beta \sum_{i=1}^q \sqrt{y_i}$, for all $x \in v^\perp$ satisfying the constraints. Hence $a_1=\beta \sum_{i=1}^q \sqrt{y_i}$.   Note that $\beta \sum_{i=1}^q \sqrt{y_i}$ is in fact an upper bound for $K(x), x \in \feas(\oR_1)$, since the term $x^\top vv^\top x$ is a square.

 From Lemma~\ref{lem:cor:transv}, we know that $\crit_0 J_1$ intersects $\Lambda_v$, and from Corollary~\ref{cor:convnormaldyn}, we know that $\crit_0 J_1 = \feas (\oR_1)$. Hence $\Lambda_v \cap \feas(\oR_1)$ is non-empty. The value of the cost at these points is upper bounded by $\beta\sum_{i=1}^q \sqrt{y_i}$, which is the value of the cost of $v^\perp$,  which proves the claim. 
\end{proof}
We set, in view of the above Proposition, $$\crit_0^* \oR_1 = \Lambda_v.$$
As a consequence of Lemmas~\ref{lem:invdyn1},~\ref{lem:cor:transv},~\ref{lem:crits1} and~\ref{lem:Lvincl}, we have shown the  following:

\begin{theorem}\label{th:compat}
	Under the tame spectrum assumption, the primal problem described by $J_1$, and the regularization problem $\oR_1$  are compatible in the sense of Def.~\ref{def:compat}, with $\crit_0^* \oR_1 = \Lambda_v$.
\end{theorem}

\paragraph{Convergence to critical points of the regularization problem} Having established that the primal and regularization problems are compatible, we now show that when initialized in $\Lambda_v$, the flow converges generically to a critical point of the regularization problem---recall that by critical point of the regularization problem, we mean a point $x \in \crit^* \oR$ that meets the constraints $\rho_i(x)=0$. We already know that all minima of the regularization problem are in $\Lambda_v$ (Lemma~\ref{lem:Lvincl}), that $\Lambda_v$ intersects $\crit_0 J_1$ transversally, and that $\feas(\oR_1) = \crit_0 J_1$ (Theorem~\ref{th:morsebott}). This is not sufficient to show convergence to $\feas(\oR_1)$ when in $\Lambda_v$ however, as the dynamics in $\Lambda_v$ can have sinks that are saddles point for the general primal dynamics. We thus show now that all the sinks of the flow restricted to $\Lambda_v$ are also in $\crit_0 J_1$; said otherwise, no locally stable critical point of the flow restricted to the invariant subspace $\Lambda_v$ is a saddle or regular point for the dynamics in $\R^m$.
\begin{theorem}\label{th:rank1}
Assume that the tame spectrum assumption holds with leading eigenvector $v$. Then the dynamics of the the normal system~\eqref{eq:nomsys}, with $x \in \R^m$, is such that  generically for $x_0 \in \Lambda_v$, $x(t)$ converges to a critical point of the regularization problem~\eqref{eq:optx}. 
In particular, all the sinks for the normal dynamics restricted to the invariant subspace $\Lambda_v$ are sinks for the normal dynamics in $\R^m$.
\end{theorem}

\begin{proof} Because $J_1$ is a Morse-Bott function by Theorem~\ref{th:morsebott}, and the intersection of its critical set with $\Lambda_v$ is transversal and of dimension $0$, the restriction of $J_1:\Lambda _v \to \R$ is a Morse function. Hence starting from $x_0$, the flow  converges to a critical point of $J_1$ and, generically, to a minimum of $J_1$.
Furthermore $x \in \Lambda_v \cap \crit J_1$ is a sink for the dynamics in $\Lambda_v$ only if (1) $x$ is a local minimum of $J_n(x)$ or (2) a saddle point of $J_1(x)$ {\it and} the Hessian of $J_1(x)$ is positive definite on $\Lambda_v$. 

We thus need to show that there are no local sink of type (2) to prove the proposition. To this end, recall from the proof of Theorem~\ref{th:morsebott} that the critical points of the gradient of $J_1(x)$ are characterized by $\rho_i(\bar x_i) \bar x_i=0$, for some $1 \leq i \leq q$, and that the Hessian of $J_1(x)$ is block diagonal. To fix ideas, consider a saddle point so that $\bar x_1 = 0$.  The leading $r_1 \times r_1$ block of the Hessian at such point is $-y_i E_i$. The line $$\lambda_1 \begin{pmatrix} \bar v_1 \\ 0\\ \vdots \\0 \end{pmatrix},$$ with $\lambda_1 \in \R$ is clearly included in $\Lambda_v$, and the Hessian of $J_1(x)$ at this saddle point, restricted to this line, is negative definite. Hence saddles so that $\bar x_1=0$ are not local sinks in the dynamics restricted to $\Lambda_v$, but saddle points as well. The same reasoning applies to any $i=1,\ldots,q$, which concludes the proof.
 \end{proof}

\paragraph{Convergence to $\crit^*_0\oR$}

We now show that when initialized near $0$, the primal flow goes arbitrarily close to the invariant space $\Lambda_v$, which we know contain all minimizers of the regularization problem.
\begin{proposition}\label{prop:winLv}
Assume that the tame spectrum assumption holds, with leading eigenvector $v \in \R^m$. Let $W \in \R^n$ be an eigenvector of $\sum_{i=1}^q y_i A_i$ associated to the largest eigenvalue, and let $w:= (P^{-1}W)_{1,\ldots,m}$, where $P$ is the matrix of Lemma~\ref{lem:P}. Then, generically for $y_i>0$, $i=1,\ldots,q$, we have  $w \in \Lambda_v$. 	
\end{proposition}
Recall that from Corollary~\ref{cor:rank1botrank1}, we know that when initialized near zero, the primal flow goes arbitrarily close to such a  $W$. The above Proposition thus says that {\it the flow in normal coordinates goes arbitrarily close to a vector $w \in \Lambda_v$.}

\begin{proof}
From  Lemma~\ref{lem:P},   $\sum_{i=1}^q y_i P^{-1}A_iP \in \R^{n \times n}$ is a block diagonal matrix with leading block $Q(\sum_{i=1}^q y_i E_i) \in \R^{m \times m}$, and other entries zero. Since $W$ is an eigenvector of $\sum_{i=1}^q y_i A_i$ associated to the largest eigenvalue,  $w$ is an eigenvector associated to the largest eigenvalue of $Q(\sum_{i=1}^q y_i E_i) \in \R^{m \times m}$. Explicitly,  $$(\alpha vv^\top + \beta I)(\sum_{i=1}^q y_i E_i)  w = \mu w,$$ for some $\mu >0$. Set $ a_1:=\sum_{i=1}^q y_i  v^\top E_i w$ and $\alpha_1:= \alpha a_1$;   we get \begin{equation}\label{eqq:condlambdawv}\alpha_1 v = (\mu I- \beta \sum_{i=1}^q y_i E_i)  w.\end{equation}
We now show that $ w \in \Lambda_v$. The proof is similar to parts of the proof of Lemma~\ref{lem:Lvincl}. By construction, $(\mu I- \beta \sum_{i=1}^q y_i E_i) \in \cD_N$. If $\alpha_1 \neq 0$, recalling that $\bar v_i \neq 0$ by assumption, we see that $(\mu I- \beta \sum_{i=1}^q y_i E_i)$ is invertible and the claim is proven. If $\alpha_1=0$,  the previous relation implies that $y_i E_i = \beta/\mu I$ or $\bar w_i =0$, for $i=1,\ldots,q$. If all $\bar w_i$ vanish, then $w=0$, which is a contradiction. Thus, assume that $\bar w_i \neq 0$ for some $1 \leq i \leq q$, then $y_i = \beta/\mu$, which is not generic for $y$.  Hence $\mu I-\beta(\sum y_i E_i) $ is generically, for $y_i >0$,  invertible, which concludes the proof
\end{proof}

The following Corollary says that when writing $w = \Lambda v$, the matrix $\Lambda$ has either all positive or all negative entries. We will  need this result in the next section.
\begin{corollary}\label{cor:lambdasign}
	Let $w$ be as in the statement of Prop.~\ref{prop:winLv}, and write $w = \Lambda v = \sum_{i=1}^q \lambda_i E_i v.$ Then $\lambda_i \leq 0$ for $i=1,\ldots,q$ or $\lambda_i \geq 0$ for  $i=1,\ldots,q$.
\end{corollary}
	\begin{proof}
		Starting from Eq.~\eqref{eqq:condlambdawv}, it is enough to show that $\mu - \beta y_i \geq 0$, for $i=1,\ldots,q$,  where $\mu$ is the largest eigenvalue of $ (\alpha vv^\top + \beta I)(\sum_{i=1}^q y_i E_i)$. Whether $\lambda_i \leq 0$ or $\lambda_i \geq 0$ is then decided by the sign of $\alpha_1$, defined above Eq.~\eqref{eqq:condlambdawv}.
		Set $D_y = \sum_{i=1}^q y_i E_i$ and denote by $D^{1/2}$ its square root. Then a short calculation shows that $(\alpha vv^\top + \beta I)D_y$ and $$R:=D_y^{1/2}(\alpha vv^\top + \beta I)D_y^{1/2}$$ have the same eigenvalues and $R$ is positive definite. Set $\bar v = D_y^{1/2}v$ and write $R=\alpha \bar v \bar v^\top + \beta D_y.$ Since $\mu$ is the largest eigenvalue of $R$, $$\mu I -(\alpha \bar v \bar v^\top + \beta D_y) \geq 0,$$i.e. it is positive semi-definite. Thus $\mu I-\beta D_y \geq \alpha \bar v \bar v^\top\geq 0$. Since $D_y$ is diagonal, with  $y_i$ on the diagonal entries,   the result is proven.
	\end{proof}
The following Proposition shows that if $x_0$ is a point in $\Lambda_v$ that converges, under the primal gradient flow, to a critical point $x^*$, then starting close enough to $x_0$ guarantees that the flow will converge to a point close to $x^*$. Note that the fact that all sinks in $\Lambda_v$ were also sinks in $\R^m$ plays a key role here: if $x^*$ were a sink in $\Lambda_v$ and a saddle for the general dynamics,  with an unstable direction necessarily outside of $\Lambda_v$, the flow lines would escape the vicinity of $\Lambda_v$ along this line. This fact is used implicitly below when appealing to the property that if $x^*$ is a sink in $\Lambda_v$, then $J_1(x^*)=0$. 
\begin{proposition}\label{prop:closeLoja}
	Assume that the tame spectrum assumption holds, with leading eigenvector  $v \in \R^m$. Let $x_0 \in \Lambda_v$ be such that $\varphi_\infty(x_0)=x^*\in \Lambda_v$, where $\varphi_t(x_0)$ is the solution of the normal dynamics $\dot x = -QDx$ at time $t$ with initial state $x_0$. Then generically for $x_0$, for all $\epsilon >0$, there exists $\delta >0$, so that for all $x_1$ with $\|x_1-x_0\|<\delta$,   $\|\varphi_{\infty}(x_1)-x^*\|<\epsilon.$	

\end{proposition}

Note that the above statement is obvious in two cases: if $x^*$, in addition to being an isolated sink in $\Lambda_v$, is also an isolated sink in $\R^m$ or if $x_1 \in \Lambda_v$ as well. 	
\begin{proof}
	From Theorem~\ref{th:rank1}, we know that generically for $x_0$, $x^* \in \crit_0 J_1$, the set of sinks for the normal dynamics in $\R^n$. From the remark above, we can assume without loss of generality that $x^*$ belongs to a connected component of $\crit_0 J_1$ of dimension larger than $0$. From Theorem~\ref{th:morsebott}, we know that the sinks are such that $\rho_i(x^*)=0$, $i=1,\ldots,q$,  which implies that $J_1(x^*)=0$.

Recall the Lojasiewicz inequality~\cite[Prop. 1, p 67]{loja1965}: for an analytic function $J_1(x)$, there exists a $\delta_1>0$ so that for all $x$ with $\|x-x^*\|<\delta_1$, there exists $\frac{1}{2} \leq \theta<1$, and a constant $c>0$ so that \begin{equation}\label{eq:lojagrad}|J_1(x)|^\theta \leq c \| \grad J_1(x)\|.\end{equation}

Now consider the solution $x(t)$ of the normal dynamics $\dot x = -\grad J_1(x)$ initialized at $x_2$ near $x^*$. Denote by $$a := \min_{x \in (\crit J_1 - \crit_0 J_1)} J_1(x).$$ That is, $a$ is the lowest value of a critical point of $J_1$ which is not a local minimum, for which we already know that $J_1 =0$. Assume, perhaps taking $x_2$ closer to $x^*$, that $J_1(x_2)<a$.  The following argument, showing that the length  of the gradient flow line starting from $x_2$ is bounded, is classical. Since $J_1(x(t)) > 0$ away from the critical set $\crit J_1$, we can write 
\begin{align*}
\frac{d}{dt} J_1(x(t))^{1-\theta} &= (1-\theta) (\grad J_1(x(t)))^\top \dot x(t) J_1(x(t))^{-\theta} \\
&= -(1-\theta) \|\grad J_1(x(t))\|^2 	J_1(x(t))^{-\theta} \\
& \geq -\frac{1-\theta}{c}  \| \grad J_1(x(t))\| \geq -\frac{1-\theta}{c} \|\dot x\|.
\end{align*}
where we used Lojasiewicz inequality to obtain the last line.
We thus have that \begin{align*}
\operatorname{Length}(x(t))= \int_0^\infty \|\dot x(t)\| dt 	& \leq -\frac{c}{1-\theta} J_1(x(t))^{1-\theta}|_0^\infty\\
 & \leq k J_1(x_2)^{1-\theta}
 \end{align*}
for some $k>0$ and where we used the fact that $J_1(\varphi_{\infty}(x_2))=0$, which is an easy consequence of the facts that $J_1(x_2)<a$ and that a gradient flow converges to its critical set. Hence, the length of trajectory of the normal system initialized at $x_2$ near $x^*$ has a length bounded by $kJ_1(x_2)^{1-\theta}$.

Since $J_1(x)$ is continuous, we can choose $0<\delta_2<\min(\delta_1,\epsilon/2)$ small enough so that  for all $x_2$ with $\|x_2-x^*\|<\delta_2$, the following two items hold: (1) $k J_1(x_2)^{1-\theta} < \varepsilon/2$ 	 and (2) $J_1(x_2)<a$ . The first item ensures that the length of the gradient flow line starting from $x_2$ is upper-bounded by $\epsilon/2$, and the second one ensures that this gradient flow line converges to $x_2^*$ such that $J(x_2^*)=0$, as discussed above.

Because $\lim_{t \to \infty} \varphi_t(x_0)=x^*$, there exists $T>0$ so that $\|\varphi_T(x_0)-x^*\|< \delta_2/2$. Furthermore, since the flow $\varphi_t(x)$ is continuous in both $t$ and $x$, there exists $\delta>0$ so that $\|\varphi_{T}(x)-\varphi_{T}(x_0)\|<\delta_2/2$ for all $x$ so that $\|x-x_0\|<\delta.$

It is now easy to see that for such $x$ so that $\|x-x_0\|<\delta$, $\lim_{t \to \infty} \varphi_t(x)$ is within $\epsilon$ of $x^*$. Indeed, by construction, for such $x$, $\|\varphi_T(x)-x^*\|< \delta_2$. Hence  $$\|\varphi_{\infty}(x) - x^*\| \leq \|\varphi_{\infty}(x)  - \varphi_{T}(x)\| + \|\varphi_{T}(x) - x^*\|,$$ and the first term is bounded by the length of the gradient flow line, which is bounded by $\epsilon/2$, and the second term is upper bounded by $\epsilon/2$ by construction. \end{proof}

\section{Positivity and convergence to a global minimum}

We have seen in the previous section that when initialized in $\Lambda_v$,  the primal flow is the gradient of a Morse function whose minima all satisfied the constraints of the regularization problem, hence trajectories converged, generically for $x_0 \in \Lambda_v$, to $\feas(\oR_1)$.  From  Lemma~\ref{lem:cor:transv}, we know that the intersection of $\crit_0 J_1$ and $\Lambda_v$ consists of $2^q$ points. Hence, the primal flow converges a priori to any one of these.  It is easy to verify that the $2^q$ sinks of the flow in $\Lambda_v$  yield different value of the cost function of the regularization problem $\oR_1$. What is perhaps the most surprising aspect of implicit regularization for matrix factorization is that the flow will converge to (near) a {\it global minima} of the regularization problem $\oR_1$. This is due, as we will see below, to the appearance of positive definite matrices with {\it positive} entries when the problem is considered in $\Lambda_v$.  We add the assumption here that $v$ has no zero entries. This assumption holds generically for the $A_i$, and could be removed at the expense of longer proofs.

In previous sections, we derived properties of the flow in $\Lambda_v$ without deriving the explicit form of the flow in that space. In this section, the proofs are more transparent in coordinates suited to the dynamics $\Lambda_v$ and thus we start by deriving the explicit form of the normal dynamics in $\Lambda_v$. To this end, we introduce the {\bf reduced variables} $z \in \R^q$, defined by removing from $x$ the repeated entries. Precisely, for $x \in \Lambda_v$ then there exists a diagonal matrix $\Lambda \in \cD_N$ so that $x = \Lambda v.$ The matrix $\Lambda$ can uniquely be written as \begin{equation}\label{eq:definvar}\Lambda=:\sum_{i=1}^q \lambda_i E_i,\end{equation} which defines the $\lambda_i$.  
 
 The reduced variables are rescaled $\lambda_i$, precisely
 \begin{equation}\label{eq:redvar} z_i:=\lambda_i \|\bar v_i\|.	
 \end{equation} Note that by the tame spectrum assumption, $\|\bar v_i\| \neq 0$ and the above is well defined.
 Introduce the following vector \begin{equation}\label{eq:defbarv} \bar v = \begin{pmatrix} \|\bar v_1\| & \cdots & \|\bar v_q \| \end{pmatrix}^\top \in \R^q
 \end{equation}
 It is a vector with positive entries. We furthermore  denote by $D_v$ the diagonal matrix $$D_{\bar v}:= \diag(\bar v) \in \R^{q \times q}, \quad D_v := \diag(v) \in \R^{m \times m}$$ Note that $D_{\bar v}$ is invertible by the tame spectrum assumption. We now express the normal dynamics in the reduced variables:

\begin{lemma}\label{lem:reddyn}
Assume that the tame spectrum assumption holds, with leading eigenvector $v \in \R^m$. Consider the normal dynamics $\dot x = -QDx$, $x \in \R^m$. Define $G \in \R^{q \times q}$ to be the positive semi-definite  matrix with entries $g_{ij} = \|\bar v_i\|\|\bar v_j\|$, i.e. \begin{equation}\label{eq:defG}G= \bar v \bar v^\top \in \R^{q \times q}.\end{equation} Then the dynamics in reduced variables is given by
\begin{equation}\label{eq:reddyn}
	\dot z = -\left( \alpha G+\beta I\right) F(z)z,
	\end{equation}where  $F(z)$ is a diagonal matrix with entries $f_i(z) =z_i^2 - y_i$
\end{lemma}

\begin{proof}
Starting from the normal dynamics, replacing $x$ by $\Lambda v$, we obtain $\dot \Lambda v = -QD\Lambda v$. Now use the fact that since $\Lambda$ is diagonal, $\Lambda v = D_v \diag(\Lambda)$, (recall that $\diag$ applied to a vector yields a diagonal matrix and when applied to a diagonal matrix, it yields a vector) and the fact that diagonal matrices commute to obtain 
\begin{equation}
\frac{d}{dt} \diag(\Lambda) =- D_v^{-1}QD_v D\diag(\Lambda) = -(\alpha D_v^{-1} vv^\top D_v + \beta I)D\diag(\Lambda),	
\end{equation}
where we used the fact that $Q=\alpha vv^\top + \beta I$.

Consider the matrix $ D_v^{-1} vv^\top D_v.$  Clearly, it is of rank $1$, and a short calculation shows that it is explicitly given by $$D_v^{-1} vv^\top D_v= \begin{pmatrix} v_1^2 & \cdots & v_m^2\\\vdots & \ddots & \vdots \\v_1^2 & \cdots & v_m^2\end{pmatrix};$$ note that it has identical rows. 
We obtain, using the explicit form of $D_v^{-1} vv^\top D_v$ just derived, shows that $$\dot \lambda_l = -\alpha\sum_{i=1}^q \sum_{j \in N_i} v_j^2 \rho_i(x) \lambda_i - \beta \rho_l(x) \lambda_l.$$
We simplify the above expression as follows: (i)  $\sum_{j \in N_i} v_j^2 \rho_i(x) \lambda_i = \| \bar v_i\|^2 \rho_i(x) \lambda_i$  and (ii) $\rho_i(x) = x^\top E_i x - y_i = v^\top \Lambda E_i \Lambda v-y_i$, and recalling that $\Lambda = \sum_{i=1}^q \lambda_i E_i$ and that $E_iE_j = 0$ if $ i \neq j$, we obtain $$\rho_i(x) = \lambda_i^2 \|\bar v_i\|^2-y_i.$$ We conclude that $$ \dot \lambda_l = -\alpha \sum_{i=1}^q \|\bar v_i\|^2 (\|\bar v_i\|^2 \lambda_i^2 -y_l) \lambda_i -\beta (\lambda_l^2 \|\bar v_l\|^2-y_i)\lambda_l.$$

Replace $\lambda_i$ by  $z_i/\|\bar v_i\|$ in the last expression to get
$$\dot z_l =- \alpha\sum_{i=1}^q \|\bar v_i\| \|\bar v_l\|(z_i^2 - y_i) z_i -\beta (z_l^2 - y_l)z_l,$$ as announced. 
\end{proof}

The matrix $\alpha G+\beta I$ is a positive definite matrix  with {\it positive entries}. The latter fact will play a role in the next section.  Thanks to the former,  it defines an inner product on $\R^q$ and so does its inverse. We use this fact in the following result, characterizing the flow the normal dynamics in $\Lambda_v$ more precisely than in the previous section.

To this end, let $J_r$ be a Morse function on $\R^q$. The {\bf index} of a  critical point $x$ of $J_r$ is defined as the number of negative eigenvalues of the Hessian of $J_r$ evaluated at $x$. Note that local minima have index zero and local maxima have index $q$. As before, we denote by $\crit J_r$ the set of critical points of $J_r$. We decompose it as $$\crit J_r = \cup_{i=0}^q \crit_i J_r$$ where $\crit_i J_r$ is the set of critical points of $J_r$ of index $i$ (this agrees with our definition of $\crit_0 J_r$ as the set of local minima of $J_r$). We have the following result:

\begin{theorem}\label{th:reddyngradient}
The reduced dynamics of Eq.~\eqref{eq:reddyn} is the gradient flow of the Morse function \begin{equation}\label{eq:defJ}J_r:=\frac{1}{4}\sum_{i=1}^q (z_i^2- y_i)^2 \end{equation} for the inner product $\langle x,y\rangle:=x^\top \left( \alpha G+\beta I\right)^{-1} y$.                          	The critical points of $J_r$ have entries in the set $\{-\sqrt{y_i},0,\sqrt{y_i})$ and  $|\crit J_r|=3^q$. Furthermore, the index of a critical point $z$ is equal to the number of zero entries in $z$, consequently $$|\crit_i J_r| = 2^{q-i}{q \choose i}.$$
	\end{theorem}

\begin{proof}
	We first observe that $\frac{\partial J_r}{\partial z_i} = (z_i^2-y_i)z_i= f_i(z)z_i$, where $f_i(z)=(z_i^2-y_i)$ is as in the statement of Lemma~\ref{lem:reddyn}. Thus \begin{equation}\label{eq:defJrz}\frac{\partial J_r}{\partial z} = F(z) z,\end{equation} from which we see that the normal dynamics in reduced coordinates is the gradient flow of $J_r$ for the inner product described in the statement of the Theorem.
	
	The critical points of $J_r$ are so that $(z_i^2-y_i)z_i=0$ or, equivalently, $$z_i \in  \{- \sqrt{y_i},0,\sqrt{y_i}\}$$ and there are $3^q$ of them as announced. To determine the index of the critical points, recall that the signature of the Hessian at a critical point is independent of the inner product~\cite{milnor2016morse}. Hence, it suffices to analyze the matrix of second derivatives of $J_r$.  It is easy to see from   Eq.~\eqref{eq:defJrz} that $\frac{\partial^2 J_r}{\partial z^2}$ is diagonal, with entries $$\frac{\partial^2 J_r}{\partial z_i^2} = 3z_i^2-y_i.$$ From the above equation, we see that the index of a critical point $z$ is precisely the number of entries of $z$ that are zero, and that there are two choices for non-zero entries. This yields the last statement of the Theorem.
	\end{proof}

As mentioned at the beginning of this section,  we know that, generically for $z_0 \in \R^q$, the reduced dynamics will converge to a point in $\crit_0 J_r$, and that there are $2^q$ such points and that they all correspond to $x$ satisfying the constraints of~\eqref{eq:optx}. However, the corresponding value of the objective function is not the same for all elements of $\crit_0 J_r$. To see this, first recall that from Lemma~\ref{lem:Lvincl}, we know that all minimizers of the regularization problem are in $\Lambda_v$. Hence, we can without loss of generality study the regularization problem in reduced coordinates.

Denote by $\mathbbm{1}$ the matrix of all one entries. The regularization problem in reduced coordinates takes following form:
\begin{proposition}\label{prop:optred}
Consider the constrained optimization problem   $$\oR_1: \min x^\top Q^{-1} x \quad \mbox{ s.t. } x^\top E_i x = y_i, \quad i=1,\ldots,q,\mbox{ and } x \in \Lambda_v$$ where $Q=\alpha vv^\top + \beta I$. Let $x^*=\Lambda^*v$ be a critical point of this problem, with $\Lambda^* = \sum_{i=1}^q \lambda_i^* E_i.$ Then $z^* = \lambda_i^* \|v_i\|$ is a critical point of
	\begin{equation}\label{eq:optz}\oR_r: \min z^\top (-\alpha' \mathbbm{1} + \beta I) z,\quad \mbox{ s. t. } z_i^2 = y_i.\end{equation}  Furthermore, the problem has 2 global minima, at $z_i = \sqrt{y_i}$, $i=1,\ldots,q,$ and $z_i = -\sqrt{y_i}$, $i=1,\ldots,q,$. 
\end{proposition}

\begin{proof}
Since $x \in \Lambda_v$, there exists $\Lambda \in \cD_N$ so that $x = \Lambda v$. Plugging this last relation in the problem $\oR_1$, it becomes $$\min v^\top \Lambda Q^{-1}  \Lambda v \quad \mbox{ s. t. } \|\bar v_i\|^2 \lambda_i^2 = y_i,  i=1,\ldots,q.$$ Recall that  $\Lambda v = D_v\diag(\Lambda)$, and $\diag(\Lambda)=\sum_{i=1}^q \sum_{j\in N_i} \lambda_i e_j$, and use the fact that $Q^{-1} = -\alpha' vv^\top + \beta' I$ for some constants $\alpha', \beta' >0$, to rewrite the cost in the above problem as $$\diag(\Lambda)^\top D_v( -\alpha' vv^\top + \beta' I)  D_v \diag(\Lambda).$$ We have that $D_v v = \begin{pmatrix} v_1^2 & \cdots & v_m^2 \end{pmatrix}^\top$, and thus
\begin{align*}
	v^\top D_v \diag(\Lambda)^\top &=  \begin{pmatrix} v_1^2 & \cdots & v_m^2 \end{pmatrix}(\sum_{i=1}^q \sum_{j\in N_i} \lambda_i e_j)\\
	&=\sum_{i=1}^q \sum_{j \in N_i} v_j^2 \lambda_i = \sum_{i=1}^q \|\bar v_i\|^2 \lambda_i\\
	&=  \bar v^\top \lambda.
\end{align*} For the second term of the cost, we have
\begin{align*}
	 \diag(\Lambda)^\top D_v D_v \diag(\Lambda)&=  v^\top \Lambda^2 v= \sum_{i=1}^q \lambda_i^2 v^\top E_i v \\
	 &= \sum_{i=1}^q \lambda_i^2 \|\bar v_i\|^2.\\
	\end{align*}
Putting the two terms together, the cost is $-\alpha' \lambda^\top \bar v \bar v^\top \lambda +\sum_{i=1}^q \lambda_i^2 \|\bar v_i\|^2.$ Replacing $\lambda_i$ by $z_i/\| \bar v_i\|$ or in matrix form $\lambda = D_{\bar v}^{-1} z$, the regularization problem becomes $\min z^\top D_{\bar v}^{-1} \bar v \bar v^\top D_{\bar v}^{-1} z + \sum_{i=1}^q z_i^2$. Since $D_{\bar v}^{-1} v$ is the vector of all ones, we get

$$\min z^\top (-\alpha' \mathbbm{1} +\beta I)z \quad \mbox{ s.t. } z_i^2 =y_i, i=1,\ldots,q,$$
as announced.
 
To prove that the global minima are such that the entries of the vector $z$ have the same sign, recall that the  $2^q$ feasible points for the problem $\oR_r$ of Eq.~\eqref{eq:optz} are so that $z_i = \pm \sqrt{y_i}$. Writing $\mathbbm{1}=ee^\top$, where $e$ is the vector of all ones, we see that in order to minimize $z^\top (-\alpha' \mathbbm{1} + \beta' I) z$, we need to maximize $|e^\top z|$, from which the statement follows.
\end{proof}

The next Proposition shows that converging to a global minimum of $\oR_r$ in $\Lambda_v$, which took place when the dynamics was constrained to the subspace of rank $1$ matrices,  implies that the primal problem has converged to a global minimum of the original regularization problem $\oR$.		

\begin{theorem}\label{th:equivnfullselecmin}
Assume that the tame spectrum assumption holds with leading eigenvector $v \in \R^m$. Then there are global minima of the regularization problem $\oR_k$,  which are of rank $1$.	
\end{theorem}
Recall that in normal coordinates, the regularization problem takes the form 	\begin{equation}
\label{eq:optx2}
\oR_k: \min_{x \in \R^{m\times k}} \tr(x^\top Q^{-1}x) \quad \mbox{s.t. } \tr(x^\top E_i x) =y_i, i=1,\ldots,q\end{equation} and that  from Proposition~\ref{prop:optnormcoord} we know that minimizers of $\oR$ are of the form $(Px, 0)^\top$, where $x$ is a minimizer of $\oR_k$ with $k=n$.
\begin{proof}
	 Recall that if $Q=\alpha vv^\top + \beta I$, with $\alpha, \beta >0$, then $Q^{-1} = -\alpha' vv^\top  + \beta' I$, with $\alpha', \beta' >0$. Furthermore, since the tame spectrum assumption implied the rank spread condition, we know that $\sum_{i=1}^q E_i = I_m$. Thus $$\sum_{i=1}^q \tr(x^\top E_i x)= \tr (x^\top x) = \sum_{i=1}^q y_i.$$
Plugging this relation into the cost, we get 
$$ x^\top Q^{-1} x = - \alpha'  \tr (x^\top vv^\top x) + \beta' \sum_{i=1}^q y_i.$$ 
We thus need to show that there is a global maximum of rank $1$ for the problem \begin{equation}\label{eq:optglobinter}\operatorname{T}_k: \max_{x \in \R^{m \times k}} \tr (x^\top vv^\top x) \quad \mbox{s.t. } \tr(x^\top E_i x) =y_i, i=1,\ldots,q.\end{equation} 
Denote by $x^j$, $j=1,\ldots, k$, the $j$th column of $x$. We will show that the above problem admits a global maximum with $x^j=0$ for $j \geq 2$. The proof goes by induction on $k$. 

We start with $k=2$, and for ease of notation, we let $x=x^1$ and $z=x^2$. The problem~\eqref{eq:optglobinter} is
$$\max (|v^\top x|^2+|v^\top z|^2) \quad \mbox{ s.t. } x^\top E_i x + z^\top E_i z = y_i, i=1,\ldots, q.$$
The terms of the cost function can be expressed as $$|v^\top x|^2 = |\sum_{i=1}^q \bar v^\top_i \bar x_i|^2,$$ and similarly for $|v^\top z|^2$, while the constraints are $x^\top E_i x + z^\top E_i z= \|\bar x_i\|^2 +\|\bar z_i\|^2.$ 
We can thus rewrite~\eqref{eq:optglobinter} as
\begin{equation}\label{eq:optglobinter2}\max (|\sum_{i=1}^q \bar v^\top_i \bar x_i|^2+|\sum_{i=1}^q \bar v^\top_i \bar z_i|^2)\quad  \mbox{ s.t. } \|\bar x_i\|^2 + \|\bar z_i \|^2 = y_i, i=1,\ldots, q\end{equation}

We claim that if a pair $x,z$ is a global maximizer of~\eqref{eq:optglobinter2}, then the terms $\bar v_i^\top \bar x_i$ are sign-consistent, for $i=1,\ldots,q$, and similarly for $\bar v_i^\top \bar z_i$.  Indeed, if $x$ (resp.\ $z$) satisfies the constraints, changing the sign of $\bar x_i$ (resp.\ $\bar z_i$) yields an $x$ (resp.\ $z$) that also does satisfy the constraints but changes the sign of $\bar v_i^\top \bar x_i$. For {\it any} $x$ satisfying the constraints, arranging the signs of $\bar x_i$ so that $\bar v^\top_i \bar x_i$ are consistent clearly increases $|\sum_{i=1}^q \bar v_i^\top \bar x_i|$, and similarly for $z$,  which proves the claim. We assume without loss of generality that all terms $\bar v_i^\top \bar x_i$ and $\bar v_i^\top \bar z_i$  are positive.

We now furthermore claim that if $x,z$ is a global maximizer of~\eqref{eq:optglobinter2}, then the pairs $\bar x_i$ and $\bar z_i$ are both aligned with each other, and aligned with $\bar v_i.$ Indeed, assume it is not the case for $\bar x_i$, $1 \leq i \leq q$,   and without loss of generality, all $\bar v_i^\top \bar x_i$, $\bar v_i^\top \bar z_i$ are positive.  Consider the map $\bar x_i \to \Theta \bar x_i$, $\Theta \in SO(|N_i|)$: keeping all other entries of $x,z$ constant, it maps a feasible point to another feasible point since $\|\Theta \bar x_i\|^2= \|\bar x_i\|^2$. This map is surjective onto the sphere of radius $\|\bar x_i\|$ and thus contains a vector aligned with $\bar v_i$ in its image.  Maximizing over $\Theta$ the quantity $\bar v_i^\top \Theta \bar x_i$, which is clearly done when $\Theta$ is such that $\Theta\bar x_i$ is aligned with $\bar v_i$, provides a feasible point with a higher cost, which proves the claim.
 We can thus exhibit global maximizers $x,z$ of problem~\eqref{eq:optglobinter} so that $$\bar x_i =\lambda_i \bar v_i \mbox{ and } \bar z_i = \mu _i \bar v_i, \quad i=1,\ldots, q,$$ for some $\mu_i, \lambda_i\geq 0$, $i=1,\ldots, q$.  Plugging this into~\eqref{eq:optglobinter2}, we have that $\lambda_i,\mu_i$ are solutions of \begin{equation}
  \label{eq:optglobinter3}
\max \left((\sum_{i=1}^q\lambda_i \|\bar v_i\|^2)^2 +(\sum_{i=1}^q\mu_i \|\bar v_i\|^2)^2\right) \quad \mbox{ s.t. } \lambda_i^2+\mu_i^2 = \frac{y_i}{\|\bar v_i\|^2}, i=1,\ldots,q.\end{equation}
  
We now claim that if the pair $\lambda,\mu \in \R^q$ is a global maximizer of~\eqref{eq:optglobinter3}, then one of the following two alternatives hold:

\begin{enumerate}
\item 	there exists a constant $c > 0$ such that \begin{equation}
\label{eq:lmuc}
\lambda_i = c \mu_i,\quad \mbox{ for } i=1, \ldots,q, \end{equation}
\item $\lambda >0$ and $\mu = 0$ or $\lambda = 0$ and $\mu >0$, where the inequalities are to be understood entrywise.
\end{enumerate}
   To see this, introduce the Lagrange multipliers $\nu_i, i=1,\ldots, q$ and differentiate the Lagrangian  of problem~\eqref{eq:optglobinter3} with respect to $\lambda_i$ and $\mu_i$. We obtain
\begin{equation}
\left\lbrace \begin{aligned}
\frac{\partial}{\partial \lambda_i}&:\sum_{i=1}^q\lambda_i \|\bar v_i\|^2 - \nu_i \lambda_i &=0\\
\frac{\partial}{\partial \mu_i}&:\sum_{i=1}^q\mu_i \|\bar v_i\|^2 - \nu_i \mu_i &=0
\end{aligned}\right.
\end{equation} Assume that $\lambda_1=0$, then $\sum_{i=2}^q \lambda_i \|\bar v_i\|^2 =0$. Since we know that $\lambda_i \geq 0$, this implies that $\lambda_i=0$ for $i=1,\ldots,q$. Hence $\lambda =0$. It is easy to see that having $\mu=0$ additionally is not a maximizer. The same holds when switching the role of $\mu$ and $\lambda$. This shows the second alternative holds. We can now assume that  $\lambda_i \neq 0,\mu_i \neq 0$ for all $1 \leq i \leq q$ (otherwise, we are back to the case above).  Solve the above equation for $\nu_i$, and we get that $$\frac{\sum_{i=1}^q\lambda_i \|\bar v_i\|^2 }{\lambda_i}=\frac{\sum_{i=1}^q\mu_i \|\bar v_i\|^2 }{\mu_i}, \quad i=1,\ldots,q.$$ Since the terms in the numerators are the same for all $i$, the ratios $\lambda_i/\mu_i$ are all the same, which proves the first alternative. In either case, this implies the global maximizer of $T_2$ is of rank $1$. 

When $\mu=0$ or $\lambda=0$, the cost in~\eqref{eq:optglobinter3} is easily seen to be $(\sum_{i=1}^q\sqrt{y_i}\|\bar v_i\|)^2$, and when $\lambda_i=c \mu_i$ for $i=1,\ldots, q$, the constraints yields $\lambda_i^2 = \frac{y_i}{(1+c^2)\|\bar v_i\|^2}$. Plugging this into the cost in~\eqref{eq:optglobinter3}, we see that the cost is the same at such points.  We conclude that there are global maximizers with $\mu_i=0$, $i=1,\ldots, q$, which proves the claim for $k=2$.

We are now done with the base case of the induction, and proceed with the induction step. Assume that there is a global maximum  $x \in \R^{m\times (k-1)}$ for $\operatorname{T}_{k-1}$ so that $x^1$ is the only non-zero column. We show that the statement holds true for $\operatorname{T}_{k}$. To this end, let $z \in \R^{m \times k}$ be a global maximum for $\operatorname{T}_k$. We have that $z$ obeys
$$\max (\sum_{j=1}^{k-1} |v^\top z^j|^2 + |v^\top z^k|^2) \quad \mbox{ s.t. } \sum_{j=1}^{k-1} (z^j)^\top E_i z^j + (z^k)^\top E_i z^k =y_i.$$ Let $\tilde z \in \R^{m \times (k-1)}$ be a global maximizer of $$\max \sum_{j=1}^{k-1} |v^\top \tilde z^j|^2   \quad \mbox{ s.t. } \sum_{j=1}^{k-1} (\tilde z^j)^\top E_i \tilde z^j  =y_i- (z^k)^\top E_i z^k.$$ We can assume  using the induction hypothesis that only $\tilde z^1$ is non-zero. Let $w \in \R^{m \times k}$ be the concatenation of $\tilde z$ and $z^k$. Then $w \in \feas(\oT_k)$ by construction, and it is also a global maximizer. Since $w$ only has two non-zero columns, it is also the solution of the problem 
$$\max |v^\top w^1|^2+|v^\top w^k|^2 \quad \mbox{ s.t. } (w^1)^\top E_i w^1 + (w^k)^\top E_i w^k = y_i, i=1,\ldots, q$$ where $w^j=0, j=2,\ldots,k-1.$ We have shown above that this problem admits a solution so that $w^k=0$. Hence, there is a global optimizer $x$ of $\oT_k$ with only one non-zero column, which concludes the proof of the Proposition.
\end{proof}

\paragraph{Convergence to the global minima of $\oR_n$}
We now argue that the primal dynamics will converge to near a global minimum of the regularization problem $\oR_1$. From Corollary~\ref{cor:rank1botrank1}, we know that for some $t_1>0$,  $U(t_1,\delta)$ is arbitrarily close to a matrix $U_1$ of rank $1$. Without loss of generality (thanks to Lemma~\ref{lem:equivXU}), we can assume that $U_1=u_1e_1^\top$, where $u_1 \in \R^n$ is an eigenvector of $\sum_{i=1}^q y_i A_i$ corresponding to the largest eigenvalue (see Corollary~\ref{cor:rank1botrank1}). We know from Proposition~\ref{prop:winLv} that $u_1 \in \Lambda_v$, and if the dynamics is initialized close to $\Lambda_v$, it converges to a point close to $\Lambda_v$ (by Proposition~\ref{prop:closeLoja}). Assuming for a moment that $X(0)$ is of rank $1$,  we thus want to show that when initialized at $u_1$, the normal dynamics (in $\R^n$) converges to a global minimum of the regularization problem $\oR_1$ in $\R^n$. From Theorem~\ref{th:equivnfullselecmin}, we know that it is also a global minimum of $\oR_n$.

 The results of the previous section guaranteed that when initialized at $u_1$, the flow will converge generically to a critical point of $\oR_1$.
We know that we can consider the system in the reduced coordinates of Eq.~\eqref{eq:redvar}, and the corresponding dynamics is given in Eq.~\eqref{eq:reddyn}. We have shown in Theorem~\ref{th:reddyngradient} that this dynamics was gradient for a Morse function $J_r$ had exactly $2^q$ local minima, one local minima per orthant. 

On the one hand, from Prop.~\ref{prop:optred}, the global minima of the regularization problem $\oR$ are in the positive orthant and negative orthant. On the other hand, from Corollary~\ref{cor:lambdasign}, we know that if we write $u_1 = \Lambda v$,  $\lambda_i$ are either all negative or all positive, $i=1,\ldots,q$, which implies that in normal coordinates, we can assume that the flow is initialized in either the positive or negative orthant. It thus suffices to show that when initialized at a small value in the positive or negative orthant, the primal dynamics will converge to the sink in that orthant. 

We can do so by exhibiting positively invariant subspace for the dynamics. We illustrate how this can be done in the case $q=2$; a similar approach applies to $q >2$.
 Recall that $Q=\alpha'G+\beta' I$ with $G=\bar v \bar v^\top$, with $\bar v \in \R^q$ a vector with {\it strictly positive} entries. Assume without loss of generality that $\|\bar v_2\| > \|\bar v_1\|=1$. We claim that the following subset of $\R^2$ is positively invariant for the gradient flow:
\begin{equation}\label{eq:defsetD}
z \in D \mbox{  if  } \left\lbrace
\begin{aligned} z_1 &>0\\
 z_2 &>0\\
 z_2 &< \gamma_1 z_1+\sqrt{y_2}\\
 z_2 &> \gamma_2 	z_1-\sqrt{y_1}
 \end{aligned}\right.
\end{equation}
for some $\gamma_1,\gamma_2>0$. We illustrate the set in Figure~\ref{fig:invset}. Note that the points $(0, \sqrt{y_1})$ and $(\sqrt{y_2},0)$ are saddle points of the dynamics.  The global minimum $(\sqrt{y_1},\sqrt{y_2})$ belongs to this set and is the only sink in this set. To verify that the set is invariant, one has to show that  the vector field, when evaluated at the set's boundary, points toward the inside of the set (which is well-defined, since the set is a closed, contractible set of codimension 0). For the sides $z_1=0, 0 \leq z_2 \leq \sqrt{y_2}$, and $z_2=0, 0 \leq z_1 \leq \sqrt{y_1}$, this is clear from the expression of the dynamics~\eqref{eq:reddyn}: when $z_1=0$ and $0 \leq z_2 \leq \sqrt{y_2}$, we see that $\dot z_2 >0$, and similarly $\dot z_1>0$ on the boundary $z_2=0, 0 \leq z_1 \leq \sqrt{y_1}$. A normal vector to the side $z_2 =\gamma_1  z_1+\sqrt{y_2}$ is the vector $\vec{n}_1=[\gamma_1, -1]^\top$. It thus suffices to verify that $\vec{n}_1^\top \grad J_n |_{z_1 \geq 0, z_2=x1+\sqrt{y_2} \geq 0}$. Taking, for example, $\gamma_1 = \|\bar v_2\|$, we obtain 
\begin{multline}\vec{n}_1^\top \grad J_n |_{ z_2=z_1+\sqrt{y_2} \geq 0}\\=(\beta  \|\bar v_2\| (\|\bar v_2\|  - 1)) z_1^3  + (3 \beta \|\bar v_2\|  \sqrt{y_2}) z_1  + (\beta \|\bar v_2\| (y_1 + 2 y_2)) z_1. \end{multline} 
When $z_1>0$, since all the coefficients are positive, the previous expression is clearly positive. A similar approach with $\gamma_2= \|\bar v_2\| +\beta/\alpha$ yields a similar result for the other boundary. In the case $q \geq 3$, the complexity of writing down the boundary of the invariant subspace increases exponentially, and we omit this here.

\begin{figure}
\centering
\includegraphics{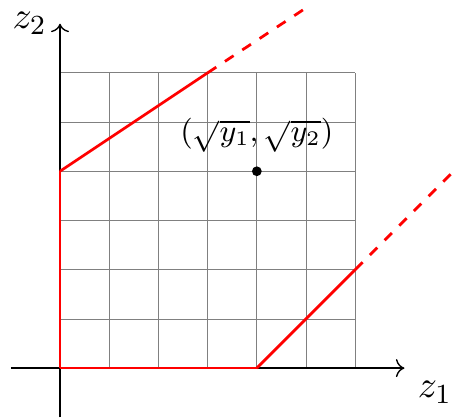}
\caption{The area inside the red boundary is invariant for the dynamics and contains a unique sink at $(\sqrt{y_1},\sqrt{y_2})$.}\label{fig:invset}
\end{figure}

\paragraph{The case of $X$ of full rank} When the primal flow is initialized {\it exactly} at a matrix $X_0$ of rank $1$, and the tame spectrum assumption holds,  using the above arguments, one can show that this dynamics will converge close to a global minimizer of the regularization problem of $\oR_n$. 

We know from the first part of the paper that, whether or not the tame spectrum assumption holds, the primal system will go arbitrarily close to a rank $1$ matrix, and that the space of rank $1$ matrices is invariant for the dynamics. However, even with the tame spectrum assumption, a result such as Prop.~\ref{prop:closeLoja} cannot be used to show that the dynamics remains close (up to $t=+\infty$) to the set of rank one matrices. The {\it additional requirement} here is that $\alpha\gg \beta$. More precisely, if the tame spectrum assumption holds and $\alpha$ is much larger than $\beta$, then we will remain close to the set of rank one matrices. The intuition behind this fact is the following: consider the normal dynamics 
$$\dot x = (\alpha vv^\top + \beta I) QD(x) x,$$
 with $x \in \R^{m \times n}$. Then clearly the term $\alpha vv^\top QD(x) x$ does not contribute to the columns of $x$ becoming more linearly independent, since its contribution to $\dot x^j$ is aligned with $v$ for all columns $x^j$. The term $\beta I  D(x) x$ can increase the rank however. Hence if $\alpha \gg 0$, the dynamics, which we know starts arbitrarily close to a rank one matrix, will converge to its equilibrium before other modes in $x$, which can arise thanks to $\beta I D(x) x$ grow large. An asymptotic analysis (for large spectral gap, i.e. $\alpha$ large) is possible, but we omit it here). We confirm this analysis in simulations.

\section{Conclusion and numerical validation}\label{sec:conclusion}
We have provided in the appendix an in-depth analysis of  implicit regularization for matrix factorization following the blueprint provided in Section~\ref{sec:impltheory}. Amongst the main findings was that under certain conditions, namely the tame spectrum condition, implicit the primal and regularization problem are compatible and approximate implicit regularization provably holds. We now discuss briefly the assumption and provide numerical evidence showing that when the tame spectrum assumption is in a sense squarely contradicted, the regularization problem and the points to which the primal flow converge seem to differ even in the limit $\delta \to 0$, where we recall that $\delta$ is the magnitude of the initial condition.

\paragraph{What else can be proved about implicit regularization for matrix factorization?} We focused in the appendix on providing a complete proof of the parts of the blueprint when the techniques involved could be applied to other settings besides matrix factorization. As such, we omitted some aspects of what would constitute a complete proof of the conjecture of~\cite{implicitregmatr2017}. Most notably, we did not provide bounds guaranteeing that when the primal system starts near a rank $1$ matrix with span in the invariant subspace  $\Lambda_v$, it converges to a point of rank $1$ close to that subspace. We emphasize again that it is here important to verify that the limit point of the trajectory of the flow, when initialized {\it near} the precritical space $\crit_0^* \oR_n$, does not leave the vicinity of that subspace. An important fact supporting this outcome is of course that there are no saddle point of the general dynamics (i.e. the dynamics not restricted to $\crit_0^* \oR_n$) that reduce to sinks (i.e. locally stable equilibria) in $\crit_0^* \oR_n$.

Besides this, as we mentioned earlier, the hypotheses can be relaxed. For example, as a consequence of the  rank spread condition,  the intersection of the range spaces of the matrices $A_i$ only contains $\{0\}$. This leaves out the trivial case of commuting, full rank (or generically any rank) matrices. One can extend the approach presented here to allow for matrices whose range spaces do not intersect trivially, but at the expense of a much heavier notation and computations. In particular, relaxing the rank spread assumption results in a version of the normal dynamics of the type $QDx$ where now $D$ is a {\it block} diagonal matrix, instead of a diagonal matrix.

Finally, we mention that the relaxation mentioned in Remark~\ref{rem:relax} may be worth exploring on its own. The conditions under which it holds are constraining, but we show in simulations (see Fig.~\ref{fig:effecttame}) that the solution we obtain is close to optimal when the assumptions are violated.

\paragraph{Numerical validation}
We present here numerical evidence supporting the conclusions made in the paper. We do not make a broad numerical study of implicit regularization for matrix factorization---we refer the reader to~\cite{implicitregmatr2017,arora2019implicit} for such studies--- but focus on addressing a few points, namely: how robust are the results when the tame spectrum assumption is not exactly met, and does implicit regularization hold when we strongly break the hypothesis?

The tame spectrum assumption can be thought of as having two characteristics: the value of $\alpha $, which is equal to the spectral gap of the matrix $\sum_{i=1}^q A_i$, and the fact that $n-1$ smallest eigenvalues of $\sum_{i=1}^q A_i$ are equal. We thus explore how the performance depends on variations in these two aspects.
To obtain the results below, we solved the ODE~\eqref{eq:mainsys} and the regularization problem~\eqref{eq:maincostU} numerically. We denote by  $\varphi_\infty(x_0)$ the point to which the ODE converged, and by $\min \oR$ the solution of the regularization problem obtained numerically. 

In order to measure the performance of the regularization problem and identify to which point the primal converges, one needs to carefully chose a metric reflecting how well the problem has been regularized, and insure that this metric can be efficiently computed. The most appropriate metric, namely $ \operatorname{dist}(\varphi_\infty(x_0)), \arg \min \oR)$, is unfortunately not easy to compute in general. Indeed, while we showed when the tame spectrum assumption holds exactly, the set $\arg \min \oR$ is essentially of cardinality two, and the two values can be computed analytically, this may not hold when the assumption is not met  exactly: the set can have high-cardinality, and numerical methods can land on an element which is far away from $\varphi_\infty(x_0)$, yet a closer element in $\arg \min\oR$ may exist. To avoid having to approximate the contents of the  $\arg \min$-set in the general case,  we instead use the average ratio 
\begin{equation}\label{eq:relerrratio}\operatorname{Relative Error}:=\frac{K(\varphi_{\infty}(x_0)) - K(\min \oR)}{K(\max \oR) - K(\min \oR)},
\end{equation}
 where we denoted by $\max \oR$ the maximal value of $K(x)$ under the constraints $\rho_i(x)=0$.\footnote{Here, we assumed that $m=n$, i.e., the $V$ variables of the normal dynamics (see Sec.~\ref{ssec:normalform}) are of the same dimension as the original variables. Clearly, if this does not hold, then $\max \oR = \infty$.}  We normalized by the difference $K(\max \oR) - K(\min \oR)$  for two reasons: first, it gives us a scale-free quantity and, perhaps more importantly, in many cases, a small difference $K(\varphi_{\infty}(x_0)) - K(\min \oR)$ misleadingly suggests that implicit regularization takes place, but in fact only reflects a set of parameters for which $K(\varphi_{\infty}(x_0)) - K(\min \oR)$ is {\it always} small.

Another aspect we investigated is the dependence of the spectrum of $\varphi_\infty(x_0)$ on the spectral gap. We mentioned at the end of the previous section, without giving a formal proof, that as $\alpha$ increases, $\varphi_\infty(x_0)$ is closer to being of rank one, and  the overall performance improves. In order to measure the distance of a rank one matrix, one could use the singular values (here, eigenvalues in fact) of $\varphi_\infty(x_0)$, but this measure is again unit dependent. We use here instead the ratio \begin{equation}\label{eq:defspectralratio}\operatorname{SpectralRatio}:=\frac{\lambda_{1}}{\sum_{i=1}^n \lambda_i},\end{equation} where $\lambda_1 \geq \cdots \geq \lambda_n$ are the eigenvalues of $\varphi_\infty(x_0)$.  Hence if $\varphi_\infty(x_0)$ is of rank $1$, the ratio above is one. In the worst case, all eigenvalues are equal and the ratio is $\frac{1}{n}$.

\paragraph{Effect of spectral gap} In a first set of simulations, we let $n=7, q=3$ and $m=7$. We sampled $N=10^4$ triplet of matrices $A_1 \in S_3,A_2 \in S_2 ,A_3\in S_2$. We took $\delta = 10^{-10}$ and solved the ODE for $T=2500$, after verifying that for a typical run, the ODE solver had converged in less than $T=50$. The initial condition is $X(0)=X_0\delta$ where $X_0= U_0U_0^\top$, and $U_0$ is sampled from a Gaussian ensemble with zero mean and unit norm. The $y_i$ where sampled from a uniform distribution with support $(0,5]$. 
We show in Fig.~\ref{fig:effectalphadeterm} the average relative error and spectral ratio as a function of the spectral gap $\alpha$, where $\beta =1 $. We see that the relative error indeed decreases rapidly as the spectral gap increases and, furthermore, performance is highly correlated with the spectral ratio as predicted.

\paragraph{Effect of equality of smaller eigenvalues}In a second set of simulations, we explore the effect of violating the tame spectrum assumption. To this end,  we sampled $N=10^4$ pair of matrices $A_1,A_2 \in S_2$, hence $m=n=4$. The eigenvalues of $\sum_{i=1}^4 A_i$ are so that $\lambda_1=1.5$, and $\lambda_2,\ldots,\lambda_4$ are sampled independently  at random from a uniform distribution with support $[1,1+\gamma]$.  Hence for $\gamma=0$, the tame spectrum assumption is met exactly, but as $\gamma$ increases, the variance in the lower eigenvalues is increased. Since for increasing $\gamma$, the spectral gap {\it decreases}, we also measured the performance for pairs $A_i$ with spectrum of the associated $Q$ being $\{1,1,1,1.5\}$ and $\{1.25,1.25,1.25,1.5\}$. We see from this experiment, see Fig.~\ref{fig:effectbeta}, that increasing the variance in the lower eigenvalues affects performance {\it minimally} when compared to the effect of the spectral gap.

\paragraph{Limit $\delta \to 0$ and tame spectrum assumption} In a last set of experiments, we investigated whether one should expect that when the tame spectrum assumption is not met, the limit as the size of $x_0$ goes to zero still implies implicit regularization. To this end, we plot the relative error as a function of $\delta$. For this case, we let $n=m=4$ and $q=2$. We let the spectrum of $\sum_{i=1}^4 A_i$ be $\{1,1,1,2\}$ and $\{1,1,2,2\}$, where we understand the second case as {\it  strongly breaking} the tame spectrum assumption. We see in Fig.~\ref{fig:effecttame} that in the former case, simulations seem to indicate that as $\delta \to 0$, the relative error indeed vanishes, whereas in the latter case, it reaches a minimum for a certain value of $\delta >0$, indicating that $\varphi_\infty(x_0)$ does not converge to the minimum of $\oR$.

\begin{figure}[h!]
\pgfplotsset{compat=1.11}
\centering
\includegraphics[scale=1]{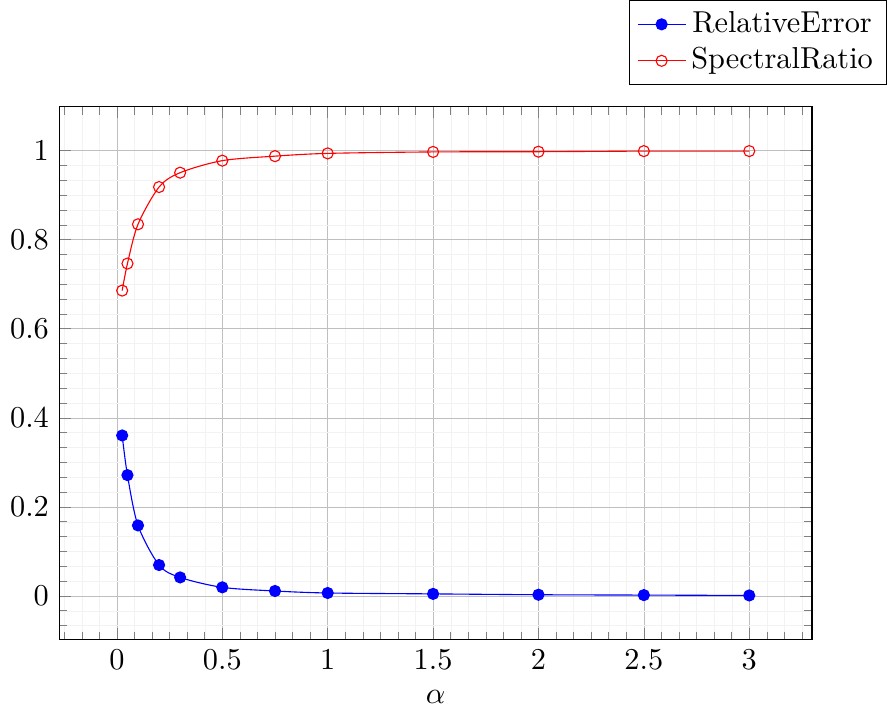}
    \caption{The relative error of Eq.~\eqref{eq:relerrratio} and spectral ratio of Eq.~\eqref{eq:defspectralratio} as a function of the spectral gap $\alpha$.}\label{fig:effectalphadeterm}
  \end{figure}

\begin{figure}[h!]
\centering
\pgfplotsset{compat=1.11}
\includegraphics{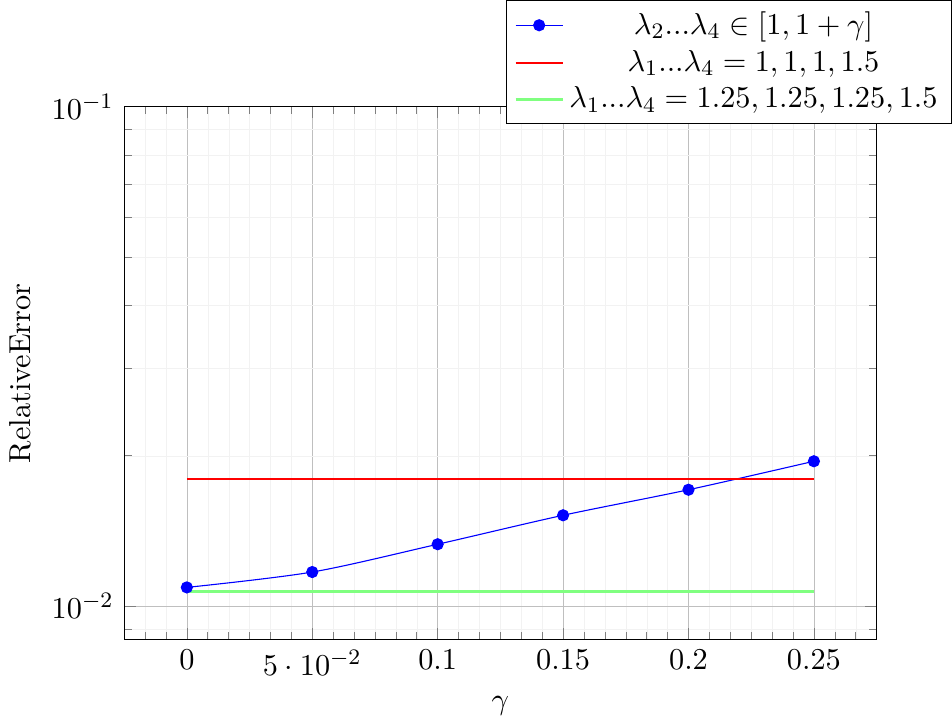}
    \caption{The relative error as a function of the dispersion of the eigenvalues $\lambda_2,\cdots,\lambda_4$. They are sampled uniformly at random from a uniform distribution with support $[1,1+\gamma]$. The eigenvalue $\lambda_1=\frac{3}{2}$. We compare with the relative error with spectrum $\{1,1,1,\frac{3}{2}\}$ and $\{\frac{5}{4},\frac{5}{4},\frac{5}{4},\frac{3}{2},\}$.}\label{fig:effectbeta}
  \end{figure}

\begin{figure}[h!]
\centering
\includegraphics{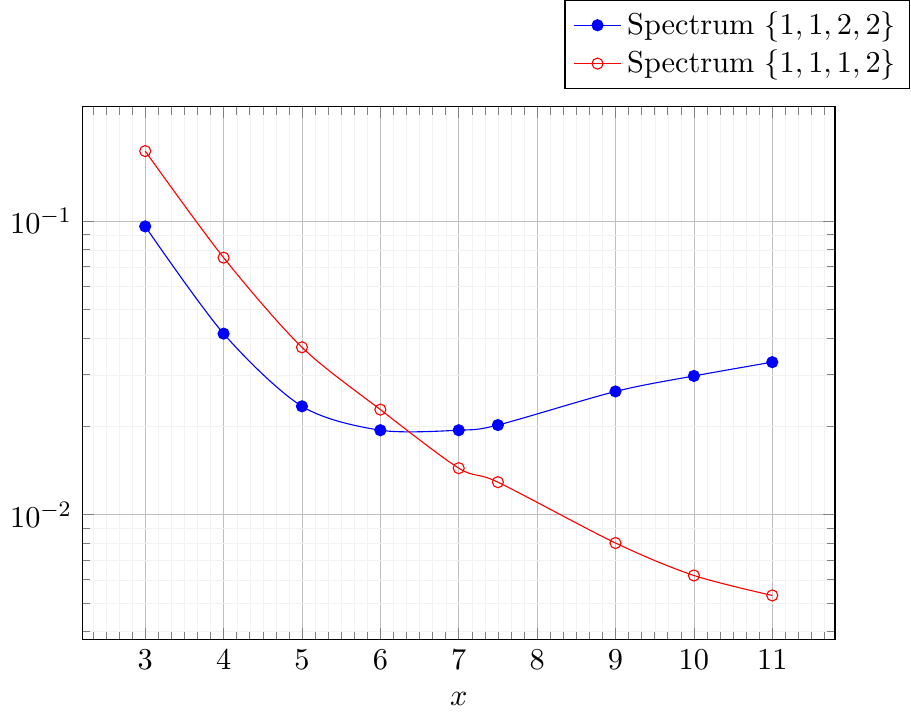}
    \caption{Relative error as a function of $\delta=10^{-x}$, where the primal flow~\eqref{eq:mainsys} is initialized at $\delta X_0$, and the spectrum of $\sum_{i=1}^q A_i$ is $\{1,1,1,2\}$ (i.e.,tame spectrum assumption met) or $\{1,1,2,2\}$.}\label{fig:effecttame}
    \end{figure}

\bibliographystyle{plain}
 \bibliography{implicitbib}

\begin{thebibliography}{10}

\bibitem{arora2019implicit}
Sanjeev Arora, Nadav Cohen, Wei Hu, and Yuping Luo.
\newblock Implicit regularization in deep matrix factorization.
\newblock {\em arXiv preprint arXiv:1905.13655}, 2019.

\bibitem{banyaga2013lectures}
Augustin Banyaga and David Hurtubise.
\newblock {\em Lectures on Morse homology}, volume~29.
\newblock Springer Science \& Business Media, 2013.

\bibitem{brockett1991dynamical}
Roger~W Brockett.
\newblock Dynamical systems that sort lists, diagonalize matrices, and solve
  linear programming problems.
\newblock {\em Linear Algebra and its applications}, 146:79--91, 1991.

\bibitem{implicitregmatr2017}
Suriya Gunasekar, Blake~E Woodworth, Srinadh Bhojanapalli, Behnam Neyshabur,
  and Nati Srebro.
\newblock Implicit regularization in matrix factorization.
\newblock In {\em Advances in Neural Information Processing Systems 30}, pages
  6151--6159, 2017.

\bibitem{helmke2012optimization}
Uwe Helmke and John~B Moore.
\newblock {\em Optimization and dynamical systems}.
\newblock Springer Science \& Business Media, 2012.

\bibitem{loja1965}
S.~Lojasiewicz.
\newblock {\em Ensembles semi-analytiques}.
\newblock IHES preprint, 1965.

\bibitem{milnor2016morse}
John Milnor.
\newblock {\em Morse theory.(AM-51)}, volume~51.
\newblock Princeton university press, 2016.

\bibitem{neyshabur2017geometry}
Behnam Neyshabur, Ryota Tomioka, Ruslan Salakhutdinov, and Nathan Srebro.
\newblock Geometry of optimization and implicit regularization in deep
  learning.
\newblock {\em arXiv preprint arXiv:1705.03071}, 2017.

\bibitem{neyshabur2014search}
Behnam Neyshabur, Ryota Tomioka, and Nathan Srebro.
\newblock In search of the real inductive bias: On the role of implicit
  regularization in deep learning.
\newblock {\em International Conference on Learning Representations, 2015},
  2014.

\bibitem{soudry2018implicit}
Daniel Soudry, Elad Hoffer, Mor~Shpigel Nacson, Suriya Gunasekar, and Nathan
  Srebro.
\newblock The implicit bias of gradient descent on separable data.
\newblock {\em The Journal of Machine Learning Research}, 19(1):2822--2878,
  2018.

\end{thebibliography}

\end{document}